%% file: arxiv_main.tex
\documentclass[11pt]{article} 

\usepackage[utf8]{inputenc} 
\usepackage[T1]{fontenc}    
\usepackage{hyperref}       
\usepackage{url}            
\usepackage{booktabs}       
\usepackage{amsfonts}       
\usepackage{nicefrac}       
\usepackage{microtype}      
\usepackage{xcolor}

\usepackage[margin=1in]{geometry}

\usepackage{natbib}

\input{preamble/packages}

\input{preamble/macro}

\title{Improved Regret Bounds for Linear \\ Bandits with Heavy-Tailed Rewards}

\date{\today}

\author{%
  Artin Tajdini \\
  University of Washington \\
  \texttt{artin@cs.washington.edu} \\
  \and
  Jonathan Scarlett \\
   National University of Singapore \\
  \texttt{scarlett@comp.nus.edu.sg} \\
  \and
  Kevin Jamieson \\
  University of Washington \\
  \texttt{jamieson@cs.washington.edu} \\
}

\begin{document}
\date{}
\maketitle

\begin{abstract}
We study stochastic linear bandits with heavy-tailed rewards, where the rewards have a finite $(1+\epsilon)$-absolute central moment bounded by $\upsilon$ for some $\epsilon \in (0,1]$. We improve both upper and lower bounds on the minimax regret compared to prior work. When $\upsilon = \mathcal{O}(1)$, the best prior known regret upper bound is $\tilde{\order}(d T^{\frac{1}{1+\epsilon}})$.  While a lower with the same scaling has been given, it relies on a construction using $\upsilon = \order(d)$, and adapting the construction to the bounded-moment regime with $\upsilon = \mathcal{O}(1)$ yields only a $\Omega(d^{\frac{\epsilon}{1+\epsilon}} T^{\frac{1}{1+\epsilon}})$ lower bound. This matches the known rate for multi-armed bandits and is generally loose for linear bandits, in particular being $\sqrt{d}$ below the optimal rate in the finite-variance case ($\epsilon = 1$).
We propose a new elimination-based algorithm guided by experimental design, which achieves regret $\tilde{\mathcal{O}}(d^{\frac{1+3\epsilon}{2(1+\epsilon)}} T^{\frac{1}{1+\epsilon}})$, thus improving the dependence on $d$ for all $\epsilon \in (0,1)$ and recovering a known optimal result for $\epsilon = 1$.  We also establish a lower bound of $\Omega(d^{\frac{2\epsilon}{1+\epsilon}} T^{\frac{1}{1+\epsilon}})$, which strictly improves upon the multi-armed bandit rate and highlights the hardness of heavy-tailed linear bandit problems. For finite action sets of size $n$, we derive upper and lower bounds of 
$\tilde{\mathcal{O}}(\sqrt d (\log n)^{\frac{\epsilon}{1+\epsilon}}T^{\frac{1}{1+\epsilon}})$ and
$\tilde\Omega(d^{\frac{\epsilon}{1+\epsilon}}(\log n)^{\frac{\epsilon}{1+\epsilon}} T^{\frac{1}{1+\epsilon}})$, respectively. 
Finally, we provide action-set-dependent regret upper bounds, showing that (i) for some geometries, such as $l_p$-norm balls for $p \le 1 + \epsilon$, we can further reduce the dependence on $d$, and (ii) for RKHS functions with the Mat\'ern kernel we can attain sublinear regret for all $\epsilon \in (0,1]$, thus substantially improving over the existing state-of-the-art.
\end{abstract}

\vspace*{-2ex}
\section{Introduction} \label{sec:intro}
\vspace*{-1ex}
\input{intro}

\vspace*{-2ex}
\section{Lower Bounds} \label{sec:lower}
\vspace*{-1ex}

\input{lowerbounds}

\vspace*{-2ex}
\section{Proposed Algorithm and Upper Bounds} \label{sec:algo}
\vspace*{-1ex}
\input{algorithm}

\section{Conclusion} \label{sec:conclusion}
In this paper, we revisited stochastic linear bandits with heavy-tailed rewards and substantially narrowed the gap between known minimax lower and upper regret bounds in both the infinite- and finite-action settings.  Our new regression estimator, guided by geometry-aware experimental design, yields improved instance-dependent guarantees that leverage the structure of the action set. Since our geometry-dependent bounds recover the $d^{\frac{2\epsilon}{1 + \epsilon}}$ dimension dependence that also appears in our minimax lower bound, we conjecture that this is the correct minimax rate for general action sets. Closing the remaining gap to establish true minimax-optimal rates for all moment parameters, and precisely characterizing the action-set-dependent complexity term under different geometries, remain promising directions for future work.

\section*{Acknowledgement}

This work was supported by the Singapore National Research Foundation (NRF) under its AI Visiting Professorship programme and  NSF Award TRIPODS 202323.

\bibliographystyle{alpha}
\bibliography{biblio}

\newpage
\appendix
\input{appendix}

\end{document}

%% file: preamble/packages.tex
\usepackage[utf8]{inputenc} 
\usepackage[T1]{fontenc}    
\usepackage{hyperref}       
\hypersetup{
    colorlinks,
    breaklinks,
    linkcolor = blue,
    citecolor = blue,
    urlcolor  = blue,
}
\usepackage{url}            
\usepackage{amsfonts}       
\usepackage{nicefrac}       
\usepackage{microtype}      
\usepackage{amsmath,bm}
\usepackage{amsthm}
\usepackage[algo2e, ruled, vlined]{algorithm2e}
\usepackage{algorithm}
\usepackage{bbm}      
\usepackage{mathtools}  \usepackage{matlab-prettifier}
\usepackage{amsopn}
\usepackage{amssymb}
\usepackage{graphicx}
\usepackage{xcolor}
\usepackage{footnote}
\usepackage{makecell}
\makesavenoteenv{tabular}
\makesavenoteenv{table}

\usepackage{booktabs,caption,threeparttable}

\usepackage{xcolor}
\usepackage{soul}
\usepackage{changepage}

\definecolor{Green}{rgb}{0.13, 0.65, 0.3}
\definecolor{Amber}{rgb}{0.3, 0.5, 1.0}
\usepackage[]{color-edits}
\usepackage{comment}

\usepackage{lipsum}
\usepackage{minitoc}
\usepackage{subcaption}

%% file: preamble/macro.tex
\newcommand{\tr}{\mathrm{Tr}}

\newcommand{\Alg}{\textit{MED-PE}}

\newcommand{\sign}{\texttt{sign}}

\sethlcolor{gray}

\newcommand{\bbR}{\mathbb{R}}

\newcommand{\bi}{\begin{itemize}}
\newcommand{\ei}{\end{itemize}}

\newtheorem{theorem}{Theorem}
\newtheorem*{theorem*}{Theorem}
\newtheorem{lemma}{Lemma}

\newtheorem{corollary}{Corollary}

\newtheorem{remark}{Remark}

\newenvironment{sketch}{%
  \proof}{\endproof}

\usepackage{amssymb}
\usepackage{pifont}%

\allowdisplaybreaks

\DeclareMathOperator*{\argmin}{\arg\!\min}
\DeclareMathOperator*{\argmax}{\arg\!\max}

\usepackage{amsmath}
\usepackage{amssymb}
\usepackage{graphicx}
\usepackage{mathtools}
\usepackage{enumerate}
\usepackage{enumitem}
\usepackage{footnote}
\usepackage{float}
\usepackage{xspace}
\usepackage{multirow}
\usepackage{nicefrac}
\usepackage{wrapfig}
\usepackage{framed}
\usepackage{url}
\usepackage{tikz}
\usetikzlibrary{arrows,chains,matrix,positioning,scopes}

\newcommand{\calA}{{\mathcal{A}}}

\newcommand{\calE}{{\mathcal{E}}}

\newcommand{\calV}{{\mathcal{V}}}

\newcommand{\field}[1]{\mathbb{#1}}

\newcommand{\E}{\field{E}}

\renewcommand{\P}{\field{P}}

\newcommand{\KL}{{\text{\rm KL}}}

\newcommand{\order}{\ensuremath{\mathcal{O}}}
\newcommand{\otil}{\ensuremath{\widetilde{\mathcal{O}}}}

\DeclareFontFamily{OMX}{MnSymbolE}{}
\DeclareFontShape{OMX}{MnSymbolE}{m}{n}{
    <-6>  MnSymbolE5
   <6-7>  MnSymbolE6
   <7-8>  MnSymbolE7
   <8-9>  MnSymbolE8
   <9-10> MnSymbolE9
  <10-12> MnSymbolE10
  <12->   MnSymbolE12}{}
\DeclareSymbolFont{mnlargesymbols}{OMX}{MnSymbolE}{m}{n}
\SetSymbolFont{mnlargesymbols}{bold}{OMX}{MnSymbolE}{b}{n}
\DeclareMathDelimiter{\llangle}{\mathopen}{mnlargesymbols}{'164}{mnlargesymbols}{'164}
\DeclareMathDelimiter{\rrangle}{\mathclose}{mnlargesymbols}{'171}{mnlargesymbols}{'171}

\usepackage{prettyref}
\newcommand{\pref}[1]{\prettyref{#1}}

\newcommand{\savehyperref}[2]{\texorpdfstring{\hyperref[#1]{#2}}{#2}}
\newrefformat{eq}{\savehyperref{#1}{Eq.~\eqref{#1}}}
\newrefformat{eqn}{\savehyperref{#1}{Equation~\eqref{#1}}}
\newrefformat{lem}{\savehyperref{#1}{Lemma~\ref*{#1}}}
\newrefformat{def}{\savehyperref{#1}{Definition~\ref*{#1}}}
\newrefformat{line}{\savehyperref{#1}{line~\ref*{#1}}}
\newrefformat{thm}{\savehyperref{#1}{Theorem~\ref*{#1}}}
\newrefformat{corr}{\savehyperref{#1}{Corollary~\ref*{#1}}}
\newrefformat{cor}{\savehyperref{#1}{Corollary~\ref*{#1}}}
\newrefformat{col}{\savehyperref{#1}{Corollary~\ref*{#1}}}
\newrefformat{sec}{\savehyperref{#1}{Section~\ref*{#1}}}
\newrefformat{app}{\savehyperref{#1}{Appendix~\ref*{#1}}}
\newrefformat{assum}{\savehyperref{#1}{Assumption~\ref*{#1}}}
\newrefformat{ex}{\savehyperref{#1}{Example~\ref*{#1}}}
\newrefformat{fig}{\savehyperref{#1}{Figure~\ref*{#1}}}
\newrefformat{tab}{\savehyperref{#1}{Table~\ref*{#1}}}
\newrefformat{medpe}{\savehyperref{#1}{Algorithm~\ref*{#1}}}
\newrefformat{alg}{\savehyperref{#1}{Algorithm~\ref*{#1}}}
\newrefformat{rem}{\savehyperref{#1}{Remark~\ref*{#1}}}
\newrefformat{conj}{\savehyperref{#1}{Conjecture~\ref*{#1}}}
\newrefformat{prop}{\savehyperref{#1}{Proposition~\ref*{#1}}}
\newrefformat{proto}{\savehyperref{#1}{Protocol~\ref*{#1}}}
\newrefformat{prob}{\savehyperref{#1}{Problem~\ref*{#1}}}
\newrefformat{claim}{\savehyperref{#1}{Claim~\ref*{#1}}}
\newrefformat{que}{\savehyperref{#1}{Question~\ref*{#1}}}
\newrefformat{obs}{\savehyperref{#1}{Observation~\ref*{#1}}}

\newcommand{\mc}[1]{\mathcal{#1}}

%% file: intro.tex
The stochastic linear bandit problem is a foundational setting of sequential decision-making under uncertainty, where the expected reward of each action is modeled as a linear function of known features. While most existing work assumes sub-Gaussian reward noise—enabling the use of concentration inequalities like Chernoff bounds—real-world noise often exhibits heavy tails, potentially with unbounded variance, violating these assumptions.
Heavy-tailed noise naturally arises in diverse domains such as high-volatility asset returns in finance \cite{ContBouchaud2000, Cont01022001}, conversion values in online advertising \cite{doi:10.1287/isre.2019.0902}, cortical neural oscillations \cite{Roberts2015}, and packet delays in communication networks \cite{10.1109/TIT.2011.2173713}. In such settings, reward distributions may be well-approximated by distributions such as Pareto, Student’s t, or Weibull, all of which exhibit only polynomial tail decay.

The statistical literature has developed several robust estimation techniques for random variables with only bounded $(1+\epsilon)$-moments (for some $\epsilon \in (0,1]$), such as median-of-means estimators \cite{devroye2015subgaussianmeanestimators, LugosiMendelson2019} and Catoni $M$-estimators \cite{Catoni2012, Brownlees_2015} in the univariate case, as well as robust least squares \cite{Audibert_2011, pmlr-v32-hsu14, han2019convergence} and adaptive Huber regression \cite{SunZhou2020} for multivariate settings.

Robustness to heavy tails was first introduced into sequential decision-making by \cite{Bub13a} in the context of multi-armed bandits. Subsequent work including \cite{MedinaYang2016,NEURIPS2018_173f0f6b, ijcai2020p406} extended these ideas to linear bandits, where each action is represented by a feature vector and the reward includes heavy-tailed noise. Generalizing robust estimators from the univariate to the multivariate setting is nontrivial, and many works have focused on designing such estimators and integrating them into familiar algorithmic frameworks like UCB.
However, the relative unfamiliarity of heavy-tailed noise can make it difficult to judge the tightness of the regret bounds. As we discuss later, this has led to some degree of misinterpretation of existing lower bounds, with key problems prematurely considered ``solved'' despite persistent, unrecognized gaps.

\subsection{Problem Statement}

We consider the problem of stochastic linear bandits with an
action set $\mathcal{A}\subseteq\mathbb{R}^d$ and an unknown parameter
$\theta^\star\!\in\!\mathbb{R}^d$.
At each round $t=1,2,\dots,T$,
the learner chooses an action $x_t\in\mathcal{A}$ and observes the reward
\[
    y_t \;=\; \langle x_t,\theta^\star\rangle \;+\; \eta_t,
\]
where $\eta_t$ are independent noise terms that satisfy $\mathbb{E}[\eta_t]=0$ and
$\mathbb{E}\bigl[|\eta_t|^{1+\epsilon}\bigr]\le\upsilon$ for some $\epsilon\in(0,1]$ and
finite $\upsilon>0$. 
We adopt the standard assumption that the expected rewards and parameters are bounded, namely, $\sup_{x \in \calA} |\langle x,\theta^\star\rangle| \le1$ and $\|\theta^\star\|_2 \le 1$.
Letting $x^\star \in \arg\max_{x\in\mathcal{A}}\langle x,\theta^\star\rangle$ be an optimal
action,
the cumulative expected regret after $T$ rounds is
\[
    R_T \;=\; \sum_{t=1}^{T} \big(
        \langle x^\star,\theta^\star\rangle
        - \langle x_t, \theta^\star \rangle \big).
        \]
Given $(\mathcal{A},\epsilon,\upsilon)$, the objective is to design a policy for sequentially selecting the points (i.e., $x_t$ for $t=1,\dotsc,T$) in order to minimize $R_T$.

\subsection{Contributions}

We study the minimax regret of stochastic linear bandits under heavy-tailed noise and make several contributions that clarify and advance the current state of the art. Although valid lower bounds exist, we show that they have been misinterpreted as matching known upper bounds. After correcting this misconception, we provide improved upper and lower bounds in the following ways:
\begin{itemize}[leftmargin=5ex]
    \item \textbf{Novel estimator and analysis:} We introduce a new estimator inspired by \cite{Cam21} (who studied the finite-variance setting, $\epsilon = 1$), adapted to the heavy-tailed setting ($\epsilon \in (0,1]$). Its analysis leads to an experimental design problem that accounts for the geometry induced by the heavy-tailed noise, which is potentially of independent interest beyond linear bandits.
    \item \textbf{Improved upper bounds:} We use this estimator within a phased elimination algorithm to obtain state-of-the-art regret bounds for both finite- and infinite-arm settings. Additionally, we derive a geometry-dependent regret bound that emerges naturally from the estimator’s experimental design.
    \item \textbf{Improved lower bounds:} We establish novel minimax lower bounds under heavy-tailed noise that are the first to reveal a dimension-dependent gap between multi-armed and linear bandit settings (e.g., when the arms lie on the unit sphere).  We provide such results for both the finite-arm and infinite-arm settings.
\end{itemize}
\pref{tab:regret_comparison} summarizes our quantitative improvements over prior work, while \pref{fig:side_by_side} illustrates the degree of improvement obtained and what gaps still remain.  

In addition to these results for heavy-tailed linear bandits, we show that our algorithm permits the kernel trick, and that this leads to regret bounds for the Mat\'ern kernel (with heavy-tailed noise) that significantly improve on the best existing bounds, in particular being sublinear for all $\epsilon \in (0,1]$.  See \pref{sec:special_cases} for a summary, and \pref{app:kernel} for the details.

\setlength{\tabcolsep}{1.5pt} 
\begin{table}[t]
\footnotesize
\centering
\caption{Comparison of regret bounds (in the $\widetilde{O}(\cdot)$ or $\widetilde{\Omega}(\cdot)$ sense) with heavy-tailed rewards for the model $y_t=\langle x_t, \theta_* \rangle + \eta_t$ where $\E[\eta_t]=0$, $\E[|\eta_t|^{1+\epsilon}] \leq 1$, $\|\theta\|_2 \leq 1$, $|\langle x, \theta \rangle| \leq 1$.  The complexity measure $M(\calA)$ is defined in \pref{thm: main-upper}.} 
\label{tab:regret_comparison}
\renewcommand{\arraystretch}{1.5}
\begin{tabular}{cccc}
\hline
\textbf{Paper} & \textbf{Setting}  & \textbf{Regret Upper Bound} & \textbf{Regret Lower Bound} \\
\hline
\cite{NEURIPS2018_173f0f6b} & general & $dT^{\frac{1}{1 + \epsilon}}$ & \multirow{2}{*}{$d^{\frac{\epsilon}{1 + \epsilon}}T^{\frac{1}{1 + \epsilon}}$\, \footnotemark} \\
\cite{huang2023tackling} & $\E[|\eta_t|^{1+\epsilon}] \leq \upsilon^{1 + \epsilon}_t$ & $d \sqrt{\sum_{t=1}^T \upsilon_t^2} T^{\frac{1 - \epsilon}{2 + 2\epsilon}}$ & \\
\hline
\cite{ijcai2020p406} & $|\calA| = n$ & $\sqrt{d \log n}T^{\frac{1}{1 + \epsilon}}$ & $d^{\frac{\epsilon}{1 + \epsilon}}T^{\frac{1}{1 + \epsilon}}$ \\
\hline
\cite{Cho19} & $\text{Mat\'ern}(\nu, d)$ & $T^{\frac{2+\epsilon}{2(1 + \epsilon)}+\frac{d}{2\nu + d}}$ & $T^{\frac{\nu + d \epsilon}{\nu(1 + \epsilon) + d \epsilon}}$  \\
\hline 
\cite{Bub13a} & MAB($\calA = \Delta^d$) & $ d^{\frac{\epsilon}{1 + \epsilon}} T^{\frac{1}{1 + \epsilon}}$& $d^{\frac{\epsilon}{1 + \epsilon}}T^{\frac{1}{1 + \epsilon}}$  \\
\hline
\multirow{5}{*}{Our Work} & $\calA$-dependent & $M(\mathcal{A})^{\frac{1}{1 + \epsilon}} \min(d, \log |\calA|)^{\frac{\epsilon}{1 + \epsilon}}T^{\frac{1}{1 + \epsilon}}$\textsubscript{(\pref{thm: main-upper})} & \\
 & general & $d^{\frac{1 + 3 \epsilon}{2(1 + \epsilon)}}T^{\frac{1}{1 + \epsilon}}$\textsubscript{(\pref{cor: gen-upper-bound})} & $ d^{\frac{2\epsilon}{1 + \epsilon}} T^{\frac{1}{1 + \epsilon}}$\textsubscript{(\pref{thm: lower-hypercube})} \\ 
& $|\calA| = n$ & $\sqrt{d} (\log n)^{\frac{\epsilon}{1 + \epsilon}}T^{\frac{1}{1 + \epsilon}}$\textsubscript{(\pref{cor: gen-upper-bound})} & $ d^{\frac{\epsilon}{1 + \epsilon}} (\log n)^{\frac{\epsilon}{1 + \epsilon}} T^{\frac{1}{1 + \epsilon}}$\textsubscript{(\pref{thm: lower-finite})} \\ 
& $\text{Mat\'ern}(\nu, d)$ & $T^{1 - \frac{\epsilon}{1 + \epsilon} \frac{2 \nu}{2 \nu + d}}$\textsubscript{(\pref{cor:Matern})} &  \\ 
& MAB($\calA = \Delta^d$) & $ d^{\frac{\epsilon}{1 + \epsilon}} T^{\frac{1}{1 + \epsilon}}$\textsubscript{(\pref{cor: mab-upper-bound})}&
\\
\hline 
\end{tabular}
\vspace{2pt}
\end{table}

\subsection{Related Work}

The first systematic study of heavy-tailed noise in bandits is due to
\cite{Bub13a},
who replaced the empirical mean in UCB with robust mean estimators, and obtained a regret bound of $\widetilde{O}\big(n^{\frac{\epsilon}{1 + \epsilon}}T^{1/(1+\epsilon)}\big)$ with $n$ arms, along with a matching lower bound. 
A sequence of follow-up works
\cite{YuNevmyvaka2018, pmlr-v97-lu19c, NEURIPS2020_607bc9eb, 9247972, DBLP:journals/corr/abs-2201-11921, chen2025uniinf}
refined these ideas and extended them to best-arm identification, adversarial, parameter-free, and Lipschitz settings.
The first extension of heavy-tailed analysis from MAB to linear bandits is due
to \cite{MedinaYang2016},
who proposed truncation- and MoM-based algorithms and proved an
$\widetilde{O}\!\bigl(d\,T^{\frac{2+\epsilon}{2(1+\epsilon)}}\bigr)$ regret bound.
Subsequently, 
\cite{NEURIPS2018_173f0f6b, ijcai2020p406} improved the regret bounds for
infinite and finite action sets, respectively (see \pref{tab:regret_comparison}). 
Huber-loss based estimators have emerged as another robustification strategy, for which
\cite{LiSun2023, pmlr-v216-kang23a, huang2023tackling, WangZhangZhaoZhou2025} provided moment-aware regret bounds. \cite{zhong2021breaking} suggested median based estimators for symmetric error distributions without any bounded moments (e.g., Cauchy).
Beyond linear bandits, \cite{XueWangWanYiZhang2023} proved a similar $dT^{\frac{1}{1 + \epsilon}}$ bound for generalized linear bandits, and \cite{Cho19} studied heavy-tailed kernel-based bandits, which we will cover in more detail in \pref{app:kernel}.  
 A summary of the best regret bounds of previous work and ours can be found in \pref{tab:regret_comparison}.
 
\begin{figure}[!tbp]
  \centering
  \begin{subfigure}[b]{0.48\textwidth}
    \centering
    \includegraphics[width=\linewidth]{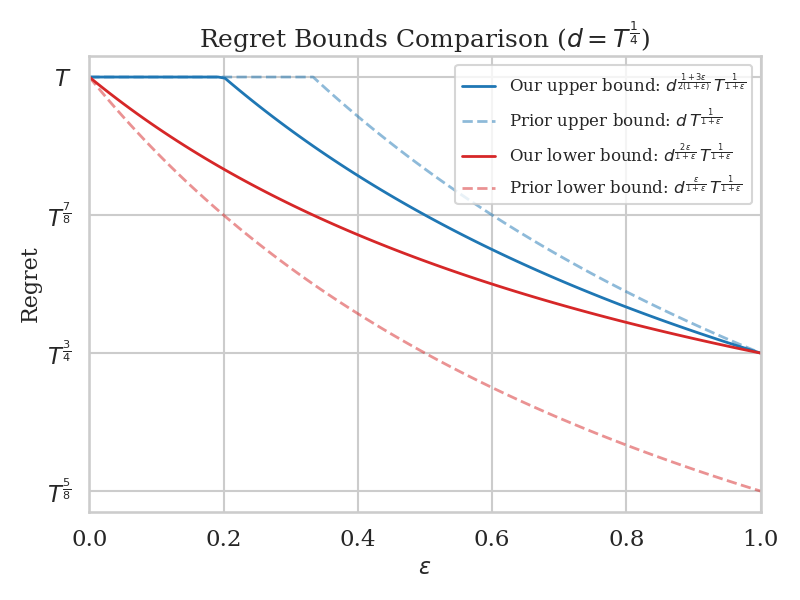}
    \caption{Regret bounds comparison}
    \label{fig:regret_comparison}
  \end{subfigure}
  \hfill
  \begin{subfigure}[b]{0.48\textwidth}
    \centering
    \includegraphics[width=\linewidth]{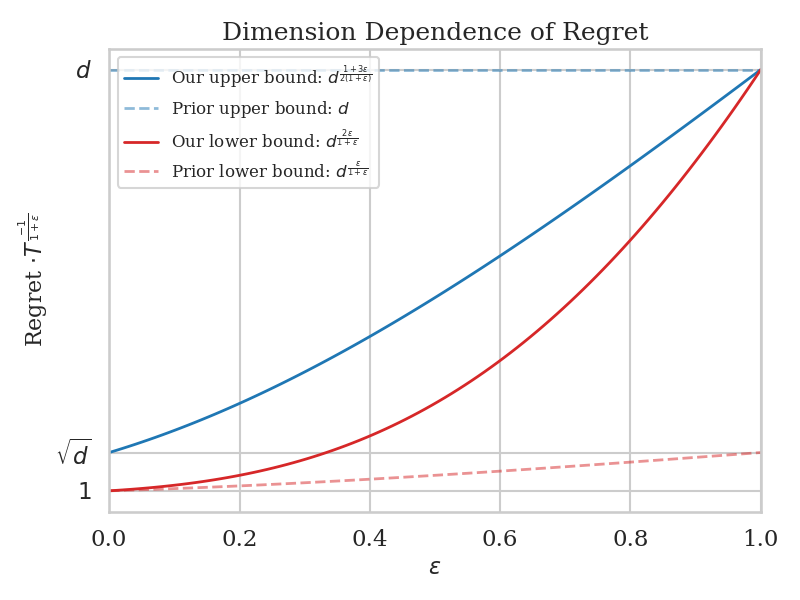}
    \caption{Dimension‐dependence comparison}
    \label{fig:dimension_dependence}
  \end{subfigure}
  \caption{
    (\subref{fig:regret_comparison}) Comparison of regret bounds across $\epsilon$ for $T = d^4$. 
    (\subref{fig:dimension_dependence}) Scaling of the bounds in $d$.
  }
  \label{fig:side_by_side}
\end{figure}

\footnotetext{We refer to this as the multi-armed bandit (MAB) rate because it matches that of a MAB problem with $d$ arms.  Note that that the $dT^{\frac{1}{1+\epsilon}}$ lower bound from \cite{NEURIPS2018_173f0f6b} was only proved for an instance with $\E[|\eta_t|^{1+\epsilon}] = O(d)$ rather than $O(1)$; see \pref{sec:lower} for further discussion.}

%% file: lowerbounds.tex
Before describing our own lower bounds, we take a moment to clarify the state of lower bounds that exist in the literature, as there has been some apparent misinterpretation within the community. 
The regret lower bound construction presented in \citep{NEURIPS2018_173f0f6b} leverages the reward distribution 
    \begin{align*}
        y(x) = \begin{cases}
            (\frac{1}{\Delta})^{\frac{1}{\epsilon}} & \text{w.p.~\,}\Delta^{\frac{1}{\epsilon}} \theta^\top x \\
            0 & \text{w.p.~\,}1 - \Delta^{\frac{1}{\epsilon}} \theta^\top x
        \end{cases}
    \end{align*}
under the choice $\Delta = \frac{1}{12} T^{-\frac{\epsilon}{1 + \epsilon}}$, and with choices of $\theta$ and $\calA$ that ensure $d\Delta \le \theta^\top x \le 2d \Delta$. 
A straightforward calculation shows that the reward distributions of this construction possesses a $(1 + \epsilon)$-absolute moment of $\Delta^{-1} (\theta^\top x) \ge d$ for all actions.
Recall that in our problem statement we consider the $(1+\epsilon)$-absolute moment to be a constant (that does not depend on the the dimension $d$ or time horizon $T$). 
We can compare this with the canonical case of sub-Gaussian noise ($\epsilon=1$) where it is assumed that the second moment is bounded by $\sigma^2=\Omega(1)$, in which it is well-known that the optimal regret rate is on the order of $\sigma d \sqrt{T}$ \cite{Csa18}.
If we were to set $\sigma^2 = \Theta(d)$, this would suggest a rate of $d^{3/2} \sqrt{T}$, but this only exceeds the usual $d\sqrt{T}$ because $\sigma$ is artificially large. 
We stress that we are not claiming that the lower bound of \citep{NEURIPS2018_173f0f6b} is in any way incorrect, and the authors even acknowledge that the bound on the moment scales with the dimension in the appendix of their work. 
We are simply pointing out that there has been some misinterpretation of the lower bound within the community.\footnote{Previous works that indicate the minimax optimality of this bound (with respect to $T$ and $d$) include \cite{ijcai2020p406, XueWangWanYiZhang2023, huang2023tackling, WangZhangZhaoZhou2025}.}

If we adjust the expected reward distributions such that $\Delta \le \theta^\top x \le 2\Delta$, so that the reward distribution maintains a constant $1 + \epsilon$ absolute moment, the resulting regret lower bound turns out to scale as $d^{\frac{\epsilon}{1+\epsilon}} T^{\frac{1}{1+\epsilon}}$,\footnote{This is obtained by optimizing $\Delta$ for the adjusted regret $\Delta 
T (\frac{1}{4} - \frac{3}{2}  \sqrt{d^{-1}\Delta^{\frac{1 + \epsilon}{\epsilon}} T})$}
matching the known optimal lower bound for the Multi-Armed Bandit (MAB) setting with $d$ arms. However, with a more precise analysis, we can prove a stronger lower bound on a similar instance (with modified parameters) having a constant $(1 + \epsilon)$-central moment of rewards, as we will see below.

\subsection{Infinite Arm Set}

Given the context above, we are ready to present our own lower bound that builds on the construction introduced by \citep{NEURIPS2018_173f0f6b} but is specifically tailored to improving the $d$ dependence.

\begin{theorem}
\label{thm: lower-hypercube}
    Fix the action set $\mathcal{A} = \{x \in [0, 1]^{2d} \,:\, x_{2i - 1} + x_{2i} = 1 \quad\forall i\in [d] \} $.  There exists a reward distribution with a $(1 + \epsilon)$-central moment bounded by $1$ and a $\theta^* \in \mathbb{R}^{2d}$ with $\|\theta^*\|_2 \le 1$ and $\sup_{x \in \calA} |x^\top \theta^*| \le 1$, such that for $T \ge 4^{\frac{1 + \epsilon}{\epsilon}}d^2$, the regret incurred is $\Omega( d^{ \frac{2\epsilon}{1 + \epsilon}} T^{\frac{1}{1 + \epsilon}})$.
\end{theorem}

\begin{proof}
    For a parameter $\Delta \le \frac{1}{4d}$ to be specified later, we let the reward distribution be a Bernoulli random variable defined as follows:
    \begin{align*}
        y(x) = \begin{cases}
            (\frac{1}{\gamma})^{\frac{1}{\epsilon}} & \text{w.p.~\,}\gamma^{\frac{1}{\epsilon}} \theta^\top x \\
            0 & \text{w.p.~\,}1 - \gamma^{\frac{1}{\epsilon}} \theta^\top x
        \end{cases}
    \end{align*}
    with $\gamma := 2d \Delta$.  We consider parameter vectors $\theta$ lying in the set $\Theta := \left\{ \theta \in \{\Delta, 2\Delta\}^{2d} : \theta_{2i - 1} + \theta_{2i} = 3\Delta  \right\}$, from which the assumption $\Delta \le \frac{1}{4d}$ readily implies $\|\theta\|_2 \le 1$ and $\sup_{x \in \calA} |x^\top \theta^*| \le 1$.  For any $\theta \in \Theta$, the $(1 + \epsilon)$-raw moment of the reward distribution (and therefore the central moment, since the rewards are nonnegative) for each action is bounded by $\E[|y(x)|^{1 + \epsilon} | x] = \gamma^{-(1+\epsilon)/\epsilon} \gamma^{1/\epsilon} \theta^\top x  = \gamma^{-1} \theta^\top x \le 1$, 
    since $\gamma = 2d\Delta$ and $\theta^\top x \le 2d\Delta$.  
     
     Let $R_T(\calA,\theta)$ be the cumulative regret for arm set $\calA$ and parameter $\theta$, and let $\text{ind}_i(\theta) := \arg\max_{b \in \{0, 1\}} (\theta_{2i - 1 + b})$ for $\theta \in \Theta$, and write $x_t = (x_{t,1},\dotsc,x_{t,d})$.  We have
    \begin{align*}
        R_T(\calA,\theta) &= \sum_{t = 1}^T \sum_{i = 1}^d \big( \Delta - \Delta x_{t, 2i - 1 + \text{ind}_i(\theta)} \big)
        = \Delta \sum_{t = 1}^T \sum_{i = 1}^d \Big( \frac{1}{2} - \frac{1}{2} (-1)^{\text{ind}_i(\theta)} (x_{t, 2i - 1} - x_{t, 2i}) \Big)
        \\ &\ge
            \frac{\Delta}{2} \sum_{i = 1}^d \E_{\theta}\bigg[ \sum_{t = 1}^T \mathbb{I}\{ (-1)^{\text{ind}_i(\theta)} (x_{t, 2i - 1} - x_{t, 2i}) \le 0 \} \bigg]
        \\ &\ge
        \frac{\Delta T}{4} \sum_{i = 1}^d \P_{\theta}\left[ \sum_{t = 1}^T \mathbb{I}\{ (-1)^{\text{ind}_i(\theta)} (x_{t, 2i - 1} - x_{t, 2i}) \le 0 \} \ge \frac{T}{2} \right],
    \end{align*}
    where the second equality follows by using $x_{t, 2i - 1} + x_{t, 2i} = 1$ and checking the cases $\text{ind}_i(\theta) = 0$ and $\text{ind}_i(\theta) = 1$ separately.

    For any $\theta \in \Theta, i \in [d]$, we define $\theta' \in \Theta$ with entries $\theta'_{j} = \begin{cases}
        3 \Delta - \theta_j & 2i - 1 \le j \le 2i \\ \theta_j &\text{otherwise}
    \end{cases}$, and let $p_{\theta, i} := \P_{\theta}\left[ \sum_{t = 1}^T \mathbb{I}\{ (-1)^{\text{ind}_i(\theta)} (x_{t, 2i - 1} - x_{t, 2i}) \le 0 \} \ge \frac{T}{2} \right]$.  We then have the following:
\begin{align*}
     p_{\theta, i} + p_{\theta', i}
    & \ge \frac{1}{2} \exp(- \KL(\P_\theta \| \P_{\theta'}) ) \tag{Bretagnolle–Huber inequality}
    \\ &= \frac{1}{2} \exp \left( - \E_\theta \left[\sum_{t = 1}^T \KL\left( \text{Ber}(\gamma^{\frac{1}{\epsilon}} \theta^\top x_t) \| \text{Ber}(\gamma^{\frac{1}{\epsilon}} \theta'^\top x_t)\right) \right]  \right). \tag{Chain rule}
\end{align*}
Now we set $\Delta := \frac{1}{2} d^{\frac{\epsilon - 1}{1 + \epsilon}} T^{\frac{-\epsilon}{1 + \epsilon}}$. Note that since $T \ge 4^{\frac{1 + \epsilon}{\epsilon}}d^2$, the above-mentioned condition $\Delta \le \frac{1}{4d}$ holds, ensuring the Bernoulli parameter is in $[0,1]$. Under this choice of $\Delta$, we have 
    \begin{align*}
        \KL\left( \text{Ber}(\gamma^{\frac{1}{\epsilon}} \theta^\top x_t) \| \text{Ber}(\gamma^{\frac{1}{\epsilon}} \theta'^\top x_t)\right) \le \frac{2^{\frac{2}{\epsilon}} 4\Delta^{\frac{2}{\epsilon}} d^{\frac{2}{\epsilon}} \Delta^2}{ 2^{\frac{1}{\epsilon}} \Delta^{\frac{1 + \epsilon}{\epsilon}} d^{\frac{1 + \epsilon}{\epsilon}} \cdot \frac{1}{2}  } = 2^{\frac{1}{\epsilon}} 8 \Delta^{\frac{1 + \epsilon}{\epsilon}} d^{\frac{1 - \epsilon}{\epsilon}} = 4 T^{-1},
    \end{align*}
    where in the first inequality we used $\KL(\text{Ber}(p) \| \text{Ber}(q)) \le \frac{(p - q)^2}{q(1 - q)}$; we get $|p-q| \le 2\gamma^{\frac{1}{\epsilon}}\Delta = 2(2d\Delta)^{\frac{1}{\epsilon}}\Delta$ because $\theta$ and $\theta'$ differ only via a single swap of $(\Delta,2\Delta)$ by $(2\Delta,\Delta)$, $q \ge \gamma^{\frac{1}{\epsilon}} \Delta d = (2d\Delta)^{\frac{1}{\epsilon}} \Delta d$ by construction, and $1-q \ge 1 - \gamma^{\frac{1}{\epsilon}}2d\Delta \ge \frac{1}{2}$ via $\Delta \le \frac{1}{4d}$.  
    
    Combining the preceding display equations gives $p_{\theta, i} + p_{\theta',i} \ge \frac{1}{2} \exp(-4)$, and averaging over  all $(\theta,\theta')$ (with $\theta' \ne \theta$) and summing over $i$, we obtain 
    $
        \frac{1}{|\Theta|} \sum_{\theta \in \Theta} \sum_{i = 1}^d p_{\theta, i} \ge \frac{1}{4} d \exp(-4).
    $
    Hence, there exists $\theta^* \in \Theta$ such that $\sum_{i=1}^d p_{\theta^*,i} \ge \frac{1}{4} d \exp(-4)$, and substituting into our earlier lower bound on $R_T$ gives $R_T(\calA, \theta^*) \ge \frac{1}{16} \exp(-4) \Delta d T = \frac{1}{32} \exp(-4) d^{\frac{2 \epsilon}{1 + \epsilon} } T^{\frac{1}{1 + \epsilon}}$.
\end{proof}
The setting in \pref{thm: lower-hypercube} is not the only one that gives regret $\Omega( d^{ \frac{2\epsilon}{1 + \epsilon}} T^{\frac{1}{1 + \epsilon}})$.  In fact, the same lower bound turns out to hold for the unit ball action set with a slight change in reward distribution to avoid large KL divergences when $\theta^\top x$ is small. The details are given in \pref{app: lower-unit-ball}.

\subsection{Finite Arm Set}

The best known lower bound for finite arm sets matches the MAB lower bound of $d^{\frac{\epsilon}{1 + \epsilon}} T^{\frac{1}{1 + \epsilon}}$ with $d$ arms (see \cite{ijcai2020p406} and the summary in \pref{tab:regret_comparison}).  We provide the first $n$-dependent lower bound (where $n := |\calA|$) by combining ideas from the MAB lower bound construction for $m$ arms with the construction used in \pref{thm: lower-hypercube} for dimension $\frac{d}{m}$, where $m^{\frac{d}{m}} \approx n$. When $n = 2^{\order(d)}$ or $n = T^{\order(d)}$, which arises naturally when finely quantizing in each dimension, our lower bound matches the infinite arm case (in the $\widetilde{\Omega}(\cdot)$ sense) as one might expect.

\begin{theorem}
    \label{thm: lower-finite}
    For each $n \in [d, 2^{\lfloor\frac{d}{4}\rfloor}]$, there exists an action set $\calA$ with $|\calA| \le n$, a reward distribution with a $(1 + \epsilon)$-central moment bounded by $1$, and a $\theta^* \in \bbR^d$ with $\|\theta^*\|_2 \le 1$ and $\sup_{x \in \calA} |x^\top \theta^*| \le 1$, such that for $T \ge 4^{\frac{1 + \epsilon}{\epsilon}}d^{\frac{1 + \epsilon}{\epsilon}}$, the regret incurred is $\Omega\big(T^{\frac{1}{1 + \epsilon}} d^{\frac{\epsilon}{1 + \epsilon}} \big(\frac{\log n}{\log d}\big)^{\frac{\epsilon}{1 + \epsilon}}\big)$.
\end{theorem}

\begin{sketch}
We outline the main proof steps here, and defer the full details to \pref{app: finite-lower}.

Consider $\log(\cdot)$ with base 2, and define $m$ to be the smallest integer such that $\frac{m}{\log m} \ge \frac{d}{\log n}$.  From the assumption $n \in [d, 2^{\lfloor\frac{d}{4}\rfloor}]$ we can readily verify that $d > 4$ and $m \in [4,d]$. For convenience, we assume that $d$ is a multiple of $m$, since otherwise we can form the construction of the lower bound with $d' = d - (d \text{ mod } m)$ and pad the action vectors with zeros. Letting $d_i := (i - 1) m$, we define the action set and the parameter set as follows for some $\Delta \le \frac{1}{4d}$ to be specified later:
\begin{gather*}
    \mathcal{A} := \bigg\{a \in \{0, 1\}^{d}: \sum_{j = d_i + 1}^{d_{i + 1}} a_{j} = 1, ~~\forall i \in [d/m] \bigg\} \\
    \theta^* \in \Theta := \left\{ \theta \in \{\Delta, 2\Delta\}^{d} : \sum_{j = d_i + 1}^{d_{i + 1}} \theta_j = (m + 1)\Delta, ~~\forall i \in [d/m]  \right\}.
\end{gather*}    
In simple terms, the $d$-dimensional vectors are arranged in $d/m$ groups of size $m$; each block in $a \in \calA$ has a single entry of 1 (with 0 elsewhere), and each block in $\theta^*$ has a single entry of $2\Delta$ (with $\Delta$ elsewhere).  The condition $\Delta \le \frac{1}{4d}$ readily implies $\|\theta^*\|_2 \le 1$ and $x^\top \theta^* \le 1$ as required.  
Moreover, we have $|\calA| = m^{\frac{d}{m}}$, and thus $\log |\calA| = \frac{d}{m} \log m \le \log n$ by the definition of $m$.

Similar to \pref{thm: lower-hypercube}, we let the rewards distribution be
    $
        y(x) = \begin{cases}
            (\frac{1}{\gamma})^{\frac{1}{\epsilon}} & \text{w.p.~\,}\gamma^{\frac{1}{\epsilon}} \theta^\top x \\
            0 & \text{w.p.~\,}1 - \gamma^{\frac{1}{\epsilon}} \theta^\top x
        \end{cases} $
with $\gamma := 2 \frac{d}{m} \Delta$.  The choices of $\calA$ and $\Theta$ give $\theta^\top x \le 2\Delta \frac{d}{m}$, so by the same reasoning as in \pref{thm: lower-hypercube}, the $(1 + \epsilon)$-moment of the reward distribution is bounded by $1$.

Let $\text{ind}_i(x) := \arg\max_{b \in [m]} (x_{d_i + b})$ for any $x \in \calA \cup \Theta$, and define $T_{i,b} := |\{t : x_{t, d_i + b} = 1\}|$. Moreover, define $t_{\rm U}$ to be a random integer drawn uniformly from $[T]$, which immediately implies that $\P_\theta[x_{t_{\rm U}, d_i + b} = 1] = \frac{\E_\theta[T_{i, b}]}{T}$.  Then we can rewrite regret as
$
        R_T(\calA,\theta) =
        \Delta T \sum_{i = 1}^{d/m} \big( 1 - \P_\theta[x_{t_{\rm U}, d_i + \text{ind}_i(\theta)} = 1] \big).
$
    For any $\theta \in \Theta$ and $i \in [\frac{d}{m}]$, and any $b \in [m]$, we define $\theta^{(b)} \in \Theta$ to have entries given by $\theta^{(b)}_{j} = \begin{cases}
        \Delta + \Delta\mathbb{I}\{j = d_i + b\} & j \in [d_i + 1, d_{i + 1}] \\ \theta_j &\text{otherwise}
    \end{cases}$; and define the base parameter $\theta^{(0)}$ with entries $\theta^{(0)}_{j} = \begin{cases}
        \Delta & j \in [d_i + 1, d_{i + 1}] \\ \theta_j &\text{otherwise}
    \end{cases}$. Note that $\theta^{(\text{ind}_i(\theta))} = \theta$. Moreover, similar to \pref{thm: lower-hypercube}, the KL divergence of reward distribution of $\theta^{(0)}$ and $\theta^{(b)}$ problem instances at action $x_t$ can be bounded by $2^{\frac{1 + \epsilon}{\epsilon}} \Delta^{\frac{1 + \epsilon}{\epsilon}} \left(\frac{d}{m}\right)^{\frac{1 - \epsilon}{\epsilon}} \mathbb{I}\{x_{t,d_i + b} = 1\}$.
    We set $\Delta := \frac{1}{8}\left(\frac{d}{m}\right)^{\frac{\epsilon - 1}{1 + \epsilon}} \left(\frac{T}{m}\right)^{\frac{-\epsilon}{1 + \epsilon}}$, from which the condition $T \ge 4^{\frac{1 + \epsilon}{\epsilon}}d^{\frac{1 + \epsilon}{\epsilon}}$ readily yields $\gamma^{\frac{1}{\epsilon}} (\frac{2d}{m} \Delta) \le \frac{1}{2}$. 
    
    Next, using Pinsker's inequality along with averaging over $b \in m$, we can show that $\frac{1}{m} \sum_b\P_{\theta^{(b)}}[x_{t, d_i + b} = 1] \le \frac{1}{m} + \frac{1}{2}$.  
    Averaging over all $\theta \in \Theta$, summing over $i \in [d/m]$, and recalling that $m \ge 4$, we obtain
    \begin{align*}
        \frac{1}{|\Theta|} \sum_{\theta \in \Theta} \sum_{i = 1}^{d/m} \big( 1 - \P_\theta[x_{t, d_i + \text{ind}_i(\theta)} = 1] \big) \ge \frac{d}{m} \Big(1 - \frac{1}{m} - \frac{1}{2}\Big) \ge \frac{d}{4m}.
    \end{align*}
    Hence, there exists $\theta^* \in \Theta$ such that $ \sum_{i = 1}^{d/m} \big( 1 - \P_{\theta^*}[x_{t, d_i + \text{ind}_i(\theta^*)} = 1]\big) \ge \frac{d}{4m}$, substituting into our earlier lower bound on $R_T$ along with our choice of $\Delta$, we obtain
    \begin{align*}
        R_T(\calA, \theta^*) \ge \frac{d}{4m} \Delta T = \frac{1}{32} d^{\frac{\epsilon}{1 + \epsilon} } \left(\frac{d}{m}\right)^{\frac{\epsilon}{1 + \epsilon}} T^{\frac{1}{1 + \epsilon}}.
    \end{align*}
    Since $f(x) = \frac{x}{\log x}$ is increasing for $x \ge e$, and $m \in [4,d]$, the definition of $m$ gives the         $\frac{d}{\log n} > \frac{m - 1}{\log(m - 1)} > \frac{m - 1}{\log m} \ge \frac{m - 1}{\log d}$. Rearranging the terms and bounding $m$, we obtain $\frac{d}{m} \ge \frac{\log n}{2\log d}$, which gives the desired result.
\end{sketch}

%% file: algorithm.tex
\begin{algorithm}[ht]
\caption{\underline{M}oment-based \underline{E}xperimental \underline{D}esign \underline{P}hased \underline{E}limination (\Alg)}
\label{alg:MEDE}
\KwIn{$\mathcal{A}$, $\gamma>0$ ,$\beta \ge 0$, $\epsilon \in (0, 1]$, $\upsilon$, $T$, robust mean estimator $\widehat{\mu}(S, \delta)$}
\textbf{Initialization} $\ell \leftarrow 1$, $t \leftarrow 0$, $\mathcal{A}_1 \leftarrow \mathcal{A}$\; \\
\While{$t < T$ and $|\calA_\ell| > 1$}{
  \begin{align*}      
    &M_{1 + \epsilon}(\lambda; \calA_\ell, \gamma, \beta) \leftarrow
      \max_{a\in\mathcal{A}_\ell}
    \mathbb{E}_{x\sim\lambda}\Big[\big|a^\top A^{(\gamma)}(\lambda)^{-1}x\big|^{1+\epsilon}\Big] + \beta^{1 + \epsilon} \|a\|^{1 + \epsilon}_{A^{(\gamma)}(\lambda)^{-1}} \tag{$A^{(\gamma)}(\lambda) := \gamma I + \E_{x \sim \lambda}[xx^\top] $}
      \\ &\lambda^*_\ell \leftarrow \argmin_{\lambda \in \Delta_{\mathcal{A}_\ell}} M_{1 + \epsilon}(\lambda; \calA_\ell, \gamma, \beta)
 \end{align*}
  
    $\varepsilon_\ell \leftarrow 2^{-\ell}, \tau_\ell \leftarrow  32^{\frac{1 + \epsilon}{\epsilon}} (1 + \upsilon)^{\frac{1}{\epsilon}} \varepsilon_\ell^{-\frac{1 + \epsilon}{\epsilon}} M_{1 + \epsilon}(\lambda^*_\ell; \calA_\ell, \gamma, \beta)^{\frac{1}{\epsilon}} \log(2\ell^2|\calA_\ell|T)$\; 

  \For{$s\leftarrow 1$ \KwTo $\tau_\ell$}{
    Draw $x_s \sim \lambda_\ell^*$, observe reward $y_s$\;
  }
  $ \displaystyle
      W^{(a)} \leftarrow \widehat{\mu}\left(\{a^\top A^{(\gamma)}(\lambda_\ell^*)^{-1}
       x_s\,y_s\}_{s = 1}^{\tau_\ell}, \frac{1}{2\ell^2T|\calA_\ell|}\right)$\; 
       \\
  $\displaystyle
    \widehat\theta_\ell \leftarrow \arg\min_{\theta} \max_{a \in \calA_\ell} |\theta^\top a - W^{(a)}|
  $\;
  
  $\displaystyle
    \mathcal{A}_{\ell+1} \leftarrow
      \bigl\{\,a\in\mathcal{A}_\ell : 
        \widehat\theta_\ell^\top a \ge
        \max_{a'\in\mathcal{A}_\ell}\widehat\theta_\ell^\top a'
        \;-\;4\varepsilon_\ell
      \bigr\}
  ,$\;
  $\ell \leftarrow \ell + 1,$\; $t \leftarrow t + \tau_\ell$}
\end{algorithm}

In this section, we propose a phased elimination–style algorithm called \Alg\, that achieves the best known minimax regret upper bound for linear bandits with noise that has bounded $(1+\epsilon)$-moments. 
In each phase $\ell$, the algorithm operates as follows:
\begin{enumerate}
    \item Design a sampling distribution over the currently active arms that minimizes the $(1 + \epsilon)$-absolute moment of a certain estimator of $\theta^*$ in the worst-case direction among all active arms (see \pref{lem: estimator-bound}), along with a suitable regularization term.
    \item Pull a budgeted number of samples (scaled by $2^{\ell \cdot \frac{1+\epsilon}{\epsilon}}$) from that distribution, and estimate the reward for each active arm separately using a robust mean estimator.
    \item Fit a parameter $\widehat\theta$ that minimizes the maximum distance of $\widehat\theta^\top a$ to the estimated reward of $a$ over all active arms.
    \item Eliminate suboptimal arms from the active set.
\end{enumerate}
This process is repeated with progressively tighter accuracy until the
time horizon is reached or a single arm remains.  In the latter case, the remaining arm is pulled for all remaining rounds.

\Alg~is a generalization of Robust Inverse Propensity Score estimator in \cite{Cam21} which assumes a bounded variance for the rewards.  We first find an experimental design that minimizes the $(1 + \epsilon)$-absolute moment of $a^\top A^{(\gamma)}(\lambda)^{-1} x$, with suitable regularization, for all $a$ that are active (and therefore the confidence interval of the robust estimator).  Note that $A^{(\gamma)}(\lambda)^{-1} x_s y_s$ (with $A^{(\gamma)}(\lambda) := \gamma I + \E_{x \sim \lambda}[xx^\top]$) can be interpreted as a single-sample regularized least squares estimator, which is then robustified through a robust mean estimation subroutine $\widehat{\mu}$ for each arm. The overall accuracy guarantee of this estimator turns out to depend directly on $M_{1 + \epsilon}(\lambda; \calA, \gamma, \beta)$ (see \pref{lem: estimator-bound} below), which is why we seek to minimize this quantity in our design $\lambda^*_{\ell}$. Moreover, we include a regularization term for design optimization to mitigate the estimator’s bias, as $A^{(\gamma)}(\lambda)^{-1} x_s y_s$ is biased for $\gamma \neq 0$.


Any robust mean estimator such as truncated (trimmed) mean, median-of-means, or Catoni's M estimator \cite{lugosi2019,Catoni2012}, can be used as the subroutine $\widehat\mu$ of \Alg\,. We adopt the truncated mean for concreteness and simplicity.
The following lemma provides our main confidence interval for our regression estimator. 

\begin{lemma} 
    \label{lem: estimator-bound}
    Consider $(x_i, y_i)_{i=1}^n$, where $x_i \sim \lambda(\calA)$ are i.i.d.~vectors from distribution $\lambda$ over $\calA$, and suppose that $y_i = \langle \theta^*, x_i \rangle + \eta_i$, where $\eta_i$ are independent zero-mean noise terms such that $\E[|\eta_i|^{1 + \epsilon}] \le \upsilon$, and $\max_{a \in \calA}|\langle \theta^*, a \rangle| \le 1$. The estimator $\widehat{\theta}(\gamma)$ with a robust mean estimator $\widehat\mu$ as a subroutine is defined as follows:
    \begin{align*}
        \widehat{\theta}(\gamma) := \arg\min_\theta \max_{a \in \calA} \left|\theta^\top a - \widehat{\mu}\left(\{a^\top A^{(\gamma)}(\lambda)^{-1}
       x_i\,y_i\}_{i = 1}^{n}, \frac{\delta}{|\calA|}\right) \right|,
    \end{align*}    
    where $A^{(\gamma)}(\lambda) := \gamma I + \E_{x \sim \lambda}[xx^\top] $. For any $\beta \geq 0$, $\widehat{\theta}(\gamma)$ with the truncated empirical mean $\widehat\mu(\{X_i\}_{i=1}^n, \delta):= \frac{1}{n} \sum X_i \mathbb{I}\big\{|X_i| \le \big( \frac{\upsilon t}{\log(\delta^{-1})}\big)^{\frac{1}{1 + \epsilon}}\big\}$ as a subroutine, satisfies the following with probability at least $1 - \delta$:
    \begin{align*}
        \max_{a \in \calA} | \langle \widehat{\theta} - \theta^*, a \rangle | \le \left( 2\gamma^{1/2} {\|\theta^*\|}_2 \beta^{-1} + 32 (1 + \upsilon)^{\frac{1}{1 + \epsilon}} \left( \tfrac{\log (|\calA|/\delta)}{n} \right)^{\frac{\epsilon}{1 + \epsilon}} \right) M_{1 + \epsilon}(\lambda; \calA, \gamma, \beta)^{\frac{1}{1 + \epsilon}},
    \end{align*}
    where $M_{1 + \epsilon}(\lambda; \calA, \gamma, \beta) :=
      \max_{a\in\mathcal{A}}
    \mathbb{E}_{x\sim\lambda}\big[\big|a^\top A^{(\gamma)}(\lambda)^{-1}x\big|^{1+\epsilon}\big] + \beta^{1 + \epsilon} \|a\|^{1 + \epsilon}_{A^{(\gamma)}(\lambda)^{-1}}$.
    
\end{lemma}

\begin{sketch}
    In order to use the robust mean estimator guaranties, we bound the $(1 + \epsilon)$-absolute moment of our samples $a^\top A^{(\gamma)}(\lambda)^{-1}
       x\,y$ for $x \sim \lambda$. Using the boundedness of the expected rewards and the $(1 + \epsilon)$-absolute moment of the noise $\eta$, we show that the moment is bounded by $4(1 + \upsilon) M_{1 + \epsilon}(\lambda; \calA, \gamma, \beta)$.
Moreover, the expected reward estimator for arm $a$ (denoted by $W^{(a)}$) is biased if $\gamma > 0$, and we can bound the bias as follows:
\begin{align*}
    \big|\langle \theta^*, a\rangle - \E[W^{(a)}]\big| \le \sqrt{\gamma} \beta^{-1} \|\theta^*\|_2 M_{1 + \epsilon}(\lambda; \calA, \gamma, \beta)^{\frac{1}{1 + \epsilon}}.
\end{align*}
Using the triangle inequality and the union bound then gives the desired result. The detailed proof is given in \pref{app: upper-bound-proof}.
\end{sketch}


The following theorem states our general action set dependent regret bound for \Alg.
\begin{theorem}
    \label{thm: main-upper}
    For any linear bandit problem with finite action set $\calA \subseteq \bbR^d$, define 
    \begin{align*}
            M^*_{1 + \epsilon}(\calA, \gamma, \beta) := \max_{\calV \subseteq \calA} \min_{\lambda \in \Delta^\calV} M_{1 + \epsilon} (\lambda; \calV, \gamma, \beta).
    \end{align*}
    If $\E[|\eta_t|^{1 + \epsilon}] \le \upsilon$, $\|\theta^*\|_2 \le b$, and $\sup_{x \in \calA} |a^\top \theta^*| \le 1$, 
    then \Alg\, with the truncated empirical mean estimator (\pref{lem: estimator-bound}) and $\gamma = T^{-\frac{2 \epsilon}{1 + \epsilon}}$ achieves regret bounded by
    \begin{align*}
    R_T \le \left( C_0\beta^{-1} b + C_1(1 + \upsilon)^{\frac{1}{1 + \epsilon}} \log(|\calA| T \log^2 T)^{\frac{\epsilon}{1 + \epsilon}} \right) M^*_{1 + \epsilon}(\calA, T^{\frac{-2\epsilon}{1 + \epsilon}}, \beta)^{\frac{1}{1 + \epsilon}} T^{\frac{1}{1 + \epsilon}}
    \end{align*} 
    for some constants $C_0$ and $C_1$.
\end{theorem}

\begin{sketch}
Using \pref{lem: estimator-bound}, with probability at least $1 - (2\ell^2T)^{-1}$, we have
\begin{align*}
\max_{a \in \calA_\ell} |a^\top\theta^* - a^\top\widehat{\theta}_\ell| \le \epsilon_\ell + 2\gamma^{1/2} b \beta^{-1} M^*_{1 + \epsilon}(\calA, \gamma, \beta)^{\frac{1}{1 + \epsilon}}.
\end{align*}
Therefore, in the phases where $\epsilon_\ell$ is large compared to $\gamma^{1/2} \beta^{-1} M^*_{1 + \epsilon}(\calA, \gamma, \beta)^{\frac{1}{1 + \epsilon}}$, suboptimal arms are eliminated, and no optimal arm is eliminated with high probability. In the phases where $\epsilon_\ell$ is smaller, each arm pull incurs regret $\otil(\gamma^{1/2} \beta^{-1} M^*_{1 + \epsilon}(\calA, \gamma, \beta)^{\frac{1}{1 + \epsilon}})$. Setting $\gamma = T^{\frac{-2\epsilon}{1 + \epsilon}}$, balances the two regret terms, and leads to the final regret bound. 
The detailed proof is given in \pref{app: upper-bound-proof}.
\end{sketch}

\begin{remark}
    If $\calA$ is not finite, 
    we can cover the domain with $T^{O(d)}$ elements in $\calA$, such that the expected reward of each arm can be approximated by one of the covered elements with $T^{-1}$ error, and therefore the bound of \pref{thm: main-upper} can be written as
    \begin{align*}
        R_T \le \left( C_0\beta^{-1} b + C'_1(1 + \upsilon)^{\frac{1}{1 + \epsilon}} d^{\frac{\epsilon}{1 + \epsilon}} \log(T^2 \log^2 T)^{\frac{\epsilon}{1 + \epsilon}} \right) M^*_{1 + \epsilon}(\calA, T^{\frac{-2\epsilon}{1 + \epsilon}}, \beta)^{\frac{1}{1 + \epsilon}} T^{\frac{1}{1 + \epsilon}}.
    \end{align*}
\end{remark}

The quantity $M^*_{1 + \epsilon}$ in \pref{thm: main-upper} may be difficult to characterize precisely in general, but the following lemma gives a universal upper bound.

\begin{lemma} \label{lem:M_bound}
    For any action set $\calA$ and $\epsilon \in (0, 1]$, setting $\gamma = T^{\frac{-2\epsilon}{1 + \epsilon}}$ and $\beta = 1$, we have 
    \begin{align*}
        M^*_{1 + \epsilon}(\calA, T^{-\frac{2 \epsilon}{1 + \epsilon}}, 1) \le d^{\frac{1 + \epsilon}{2}}.
    \end{align*}
    Moreover, a design $\lambda$ with $M_{1 + \epsilon}(\lambda; \calA, T^{\frac{- 2 \epsilon}{1 + \epsilon}}, 1) = O(d^{\frac{1 + \epsilon}{2}})$ can be found with $O(d \log \log d)$ time. 
\end{lemma}

\begin{proof}
    We upper bound the first term in the objective function as follows:
    \begin{align*}
        \E\Big[ \big| a^\top A^{(\gamma)}(\lambda)^{-1} x  \big|^{1 + \epsilon} \Big] 
      &\le  \E\Big[ \big| a^\top A^{(\gamma)}(\lambda)^{-1} x  \big|^2 \Big]^{\frac{1 + \epsilon}{2}} \tag{Jensen's inequality} 
      \\
    &= \E\big[ a^\top A^{(\gamma)}(\lambda)^{-1} x x^\top  A^{(\gamma)}(\lambda)^{-1} a \big]^{\frac{1 + \epsilon}{2}} \\
    &\le \|a\|_{A^{(\gamma)}(\lambda)^{-1}}^{1 + \epsilon}. \tag{$\E[xx^\top] = \sum_x \lambda(x) xx^\top \preceq A^{(\gamma)}(\lambda)$}
    \end{align*}
    Hence, the minimization of $M_{1+\epsilon}$ is upper bounded in terms of a minimization of $\max_a \|a\|_{A(\lambda)^{-1}}$.  This is equivalent to G-optimal design which is well-studied and the following is known (e.g., see \citep[Chapter 21]{Csa18}): (i) The problem is convex and its optimal value is at most $\sqrt{d}$; (ii) There are efficient algorithms such as Frank–Wolfe that can find a design having $\max_{a} \|a\|_{A(\lambda)^{-1}} = O(\sqrt{d})$ with $O(d \log \log d)$ iterations. 
\end{proof}

Combining \pref{thm: main-upper} and \pref{lem:M_bound}, we obtain the following. 

\begin{corollary}
    \label{cor: gen-upper-bound}
    For any action set $\calA$, \Alg\, achieves regret $\otil(d^{\frac{1 + 3\epsilon}{2(1 +\epsilon)}} T^{\frac{1}{1 + \epsilon}})$. Moreover, for a finite action set with $|\calA| = n$, the regret bound is lowered to $\otil(\sqrt{d} T^{\frac{1}{1 + \epsilon}} (\log n)^{\frac{\epsilon}{1 + \epsilon}})$.
\end{corollary}

\paragraph{Computational complexity.}
By \pref{lem:M_bound}, a design over general action sets can be computed efficiently. The truncated sample-mean estimator can also be computed in linear time. Moreover, the minimum-distance estimator for $\widehat\theta$ is obtained by solving a linear optimization problem and is therefore computable in polynomial time; in the infinite-dimensional case this is handled via a dual formulation (see \pref{app:kernel}). The dominant per-round cost is the linear pass over the active arms to update estimates and apply elimination tests, which is standard for finite-arm algorithms.

The bound in \pref{cor: gen-upper-bound} is the worst-case regret over all possible action sets $\calA$.  However, based on geometry of the action set, we can achieve tighter regret bounds, as we see below.

\subsection{Special Cases of the Action Set} \label{sec:special_cases}

\paragraph{Simplex.} When $\calA$ is the simplex, the problem is essentially one of multi-armed bandits with $d$ arms. Consider $\lambda$ being uniform over canonical basis; then $A(\lambda) = \frac{1}{d}I$, and for each $a \in \calA$, we have 
\begin{align*}
\E_{x \sim \lambda}[ |a^\top A^{-1} x |^{1 + \epsilon} ] &= \E_{x \sim \lambda}[ |d a^\top x |^{1 + \epsilon} ] = d^{1 + \epsilon} \sum_{i = 1}^d d^{-1} |a^\top e_i|^{1 + \epsilon} 
= d^\epsilon \sum_{i = 1}^d |a_i|^{1 + \epsilon} \le d^{\epsilon}.
\end{align*}
Since one of the canonical basis vectors (or its negation) must be optimal when $\mathcal{A}$ is the simplex, we can simply restrict to this subset of $2d$ actions, giving the following corollary, which recovers the well-known scaling for heavy-tailed MAB \cite{Bub13a}.

\begin{corollary} 
\label{cor: mab-upper-bound}
    For the simplex action set $\calA = \Delta^d$, if the assumptions of \pref{thm: main-upper} hold, then \Alg, with parameters $\gamma = T^{-1}, \beta = d^{\frac{\epsilon - 1}{2}}$ achieves regret $\otil(d^{\frac{\epsilon}{1 + \epsilon}} T^{\frac{1}{1 + \epsilon}} + \sqrt{dT})$.
\end{corollary}

\paragraph{$l_p$-norm ball with radius $r$ for $p \le 1 + \epsilon$.} Similarly to the simplex, if we define $\lambda$ to be uniform over $\{r \mathbf{e}_i\}_{i=1}^d$, then $A(\lambda) = \frac{r^2}{d} I$ for any $v \in \mathcal{B}(\|\cdot\|_p, r)$, and we have
\begin{align*}
\E_{x \sim \lambda}[ |a^\top A^{-1} x |^{1 + \epsilon} ] &= \E_{x \sim \lambda}\Big[ \Big|\frac{d}{r^2} a^\top x \Big|^{1 + \epsilon} \Big] = d^\epsilon \sum_{i = 1}^d \left|\frac{a_i}{r}\right|^{1 + \epsilon} 
\le d^\epsilon \sum_{i = 1}^d \left|\frac{a_i}{r}\right|^{p} \le d^{\epsilon},
\end{align*}
where the last inequality is by the definition of the $l_p$-norm ball.
\begin{corollary}
    For the action set $\calA = \{x: \|x\|_{p} \le r\}$ with $p \le 1+\epsilon$, if the assumptions of \pref{thm: main-upper} hold, then \Alg, with parameters $\gamma = T^{-1}, \beta = d^{\frac{\epsilon - 1}{2}}$, has regret of $\otil(d^{\frac{2\epsilon}{1 + \epsilon}} T^{\frac{1}{1 + \epsilon}} + d\sqrt{T})$.
\end{corollary}

\paragraph{Mat\'ern Kernels.} Our algorithm does not require the action features to lie in a finite-dimensional space, as long as the design and the estimator $a^\top A^{(\gamma)}(\lambda)^{-1} x$ can be computed efficiently. In particular, following the approach of \cite{Cam21}, our method extends naturally to kernel bandits, where the reward function belongs to a reproducing kernel Hilbert space (RKHS) associated with a kernel $K$ satisfying $K(x, y) = \phi(x)^\top \phi(y)$ for some (possibly infinite-dimensional) feature map $\phi$.  Since our focus is on linear bandits, we defer a full description of the kernel setting to \pref{app:kernel}, where we also establish the following corollary (stated informally here, with the formal version deferred to \pref{app:kernel}).
 
 \begin{corollary} \label{cor:Matern}
    {\em (Informal)}
     For the kernel bandit problem with domain $[0,1]^d$ for a constant value of $d$, under the Mat\'ern kernel with smoothness parameter $\nu > 0$, the kernelized version of \Alg\, (with suitably-chosen parameters) achieves regret $\otil(T^{1 - \frac{\epsilon}{1+\epsilon} \cdot \frac{ 2 \nu}{ 2 \nu + d}})$.
 \end{corollary}

 While this does not match the known lower bound (except when $\epsilon = 1$ or in the limit as $\epsilon \to 0$), it significantly improves over the best existing upper bound \cite{Cho19}, which is only sublinear in $T$ for a relatively narrow range of $(\epsilon,d,\nu)$.  In contrast, our bound is sublinear in $T$ for all such choices.

%% file: appendix.tex
\section{Upper Bound Proofs}
\label{app: upper-bound-proof}

\subsection{Proof of \pref{lem: estimator-bound} (Confidence Interval)}

We first state a well known guarantee of the truncated mean estimator.

\begin{lemma}\textbf{(Lemma 1 of \cite{Bub13a})}
\label{lem: trunk-mean}
Let $X_1, \ldots, X_n$ be i.i.d. random variables such that $\E[ |X_i|^{1 + \epsilon} ] \le u$ for some $\epsilon \in (0, 1]$. Then the truncated empirical mean estimator $\widehat\mu(\{X_i\}_{i=1}^n, \delta):= \frac{1}{n} \sum_{i=1}^n X_i \mathbb{I}\big\{|X_i| \le \big( \frac{u t}{\log(\delta^{-1})}\big)^{\frac{1}{1 + \epsilon}}\big\}$ satisfies with probability at least $1 - \delta$ that
\begin{align*}
        |\widehat{\mu}(\{X_i\}_{i=1}^n, \delta) - \mu| &\le 4u^{\frac{1}{1 + \epsilon}}\left(\frac{ \log(\delta^{-1})}{n}\right)^{\frac{\epsilon}{1 + \epsilon}}.
\end{align*}
\end{lemma}

Let $W^{(a)} := \widehat{\mu}\left(\{a^\top A^{(\gamma)}(\lambda)^{-1}
       x_i\,y_i\}_{i = 1}^{n}, \frac{\delta}{|\calA|}\right)$.
    We first observe that
    \begin{align*}
        \max_{a \in \calA} |a^\top \widehat{\theta}(\gamma) - a^\top\theta^*|   &= \max_{a \in \calA}|a^\top \widehat{\theta}(\gamma) - W^{(a)} + W^{(a)} - a^\top\theta^*|
        \\ & \le \max_{a \in \calA} |a^\top \widehat{\theta}(\gamma) - W^{(a)}| + \max_{a \in \calA} |W^{(a)} - a^\top\theta^*| 
        \\ &= 
        \min_\theta \max_{a \in \calA} |a^T\theta - W^{(a)}| + \max_{a \in \calA} |W^{(a)} - a^\top\theta^*|  \tag{def.~$\widehat{\theta}(\gamma)$} \\
        &\le 2\max_{a \in \calA} |W^{(a)} - a^\top\theta^*|.
    \end{align*}
    For fixed $a$, we bound the $(1 + \epsilon)$-moment of $a^\top A^{(\gamma)}(\lambda)^{-1} x y$, where $x \sim \lambda$ and $y = x^\top \theta^* + \eta$, as follows:
    \begin{align*}
        \E\Big[ \big| a^\top A^{(\gamma)}(\lambda)^{-1} x y \big|^{1 + \epsilon} \Big] &= \E\Big[ \big| a^\top A^{(\gamma)}(\lambda)^{-1} x (  x^\top \theta^* + \eta ) \big|^{1 + \epsilon} \Big] \\
        &\le 2^{1 + \epsilon} \E\Big[ \big| a^\top A^{(\gamma)}(\lambda)^{-1} x (  x^\top \theta^*) \big|^{1 + \epsilon} \Big] + 2^{1 + \epsilon}\E\Big[ \big| a^\top A^{(\gamma)}(\lambda)^{-1} x  \big|^{1 + \epsilon} |\eta|^{1 + \epsilon} \Big]  \tag{$|a+b| \le 2\max\{|a|,|b|\}$} \\
        &\leq 4 \E\Big[ \big| a^\top A^{(\gamma)}(\lambda)^{-1} x  \big|^{1 + \epsilon} \Big] + 4\upsilon \E\Big[ \big| a^\top A^{(\gamma)}(\lambda)^{-1} x  \big|^{1 + \epsilon} \Big] \tag{ $|x^\top \theta^*| \le 1$ and $\E[|\eta|^{1+\epsilon}] \le \upsilon$ }\\
        &\leq 4(1 + \upsilon) M_{1 + \epsilon}(\lambda; \calA, \gamma, \beta). \tag{def.~$M_{1+\epsilon}$}
    \end{align*}

Using this moment bound and \pref{lem: trunk-mean}, for any $a$, we have with probability at least $1 - \frac{\delta}{|\calA|}$ that
\begin{align*}
    |W^{(a)} - \E[W^{(a)}]| \le 16 (1 + \upsilon)^{\frac{1}{1 + \epsilon}} M_{1 + \epsilon}(\calA, \gamma, \beta)^{\frac{1}{1 + \epsilon}}  \left( \frac{\log (\delta^{-1}|\calA|)}{n} \right)^{\frac{\epsilon}{1 + \epsilon}}.
\end{align*}
Moreover, we have
\begin{align*}
     |a^\top \theta^* - \E[W^{(a)}]| &=
     | \langle \theta^*, a \rangle - \E[ a^\top A^{(\gamma)}(\lambda)^{-1} x x^\top \theta^* ]| \tag{def.~$W^{(a)}$} \\ &=  | \langle \theta^*, a \rangle -  a^\top A^{(\gamma)}(\lambda)^{-1} A(\lambda) \theta^* | \tag{where $A(\lambda) = \E[xx^T]$} \\
     &=  | \langle \theta^*, a \rangle -  a^\top A^{(\gamma)}(\lambda)^{-1} 
     \big( A^{(\gamma)}(\lambda) - \gamma I\big) \theta^* | \tag{$A(\lambda) = A^{(\gamma)}(\lambda) - \gamma I$} \\
     &=  \gamma |  a^\top A^{(\gamma)}(\lambda)^{-1} \theta^* | \\
     &=  \gamma |  a^\top  (A(\lambda) + \gamma I)^{-1/2} (A(\lambda) +  \gamma I)^{-1/2} \theta^* | 
     \\
     &\leq  \gamma \| a \|_{A^{(\gamma)}(\lambda)^{-1}} \gamma^{-1/2} \|\theta^*\|_{(I + \gamma^{-1}A(\lambda) )^{-1}} \tag{Cauchy–Schwarz}
     \\ &\le \gamma^{1/2} \| a \|_{A^{(\gamma)}(\lambda)^{-1}}  \|\theta^*\|_2 \tag{$I + \gamma^{-1}A(\lambda) \succeq I$}
     \\ & \le \gamma^{1/2} \beta^{-1} \|\theta^*\|_2 M_{1 + \epsilon}(\lambda; \calA, \gamma, \beta)^{\frac{1}{1 + \epsilon}}. \tag{def.~$M_{1+\epsilon}$}
\end{align*}

Putting the two inequalities together, and using the union bound completes the proof.

\subsection{Proof of \pref{thm: main-upper} (Regret Bound for \Alg)}

Using \pref{lem: estimator-bound} for action set $\calA_\ell$, we have with probability of at least $1 - \frac{1}{2\ell^2T}$,
\begin{align*}
\max_{a \in \calA_\ell} |a^\top\theta^* - a^\top\widehat{\theta}_\ell| &\le 
\left( 2\gamma^{1/2} {\|\theta^*\|}_2 \beta^{-1} + 32 (1 + \upsilon)^{\frac{1}{1 + \epsilon}} \left( \frac{\log (2l^2T|\calA_\ell|)}{\tau_\ell} \right)^{\frac{\epsilon}{1 + \epsilon}} \right) M_{1 + \epsilon}(\lambda^*_\ell; \calA_\ell, \gamma, \beta)^{\frac{1}{1 + \epsilon}}
\\ &\le 2\gamma^{1/2} b \beta^{-1} M_{1 + \epsilon}(\lambda^*_\ell; \calA_\ell, \gamma, \beta)^{\frac{1}{1 + \epsilon}} + \epsilon_\ell \tag{choice of $\tau_\ell$ in \pref{alg:MEDE}} \\
&\le 2\gamma^{1/2} b \beta^{-1} M^*_{1 + \epsilon}(\calA, \gamma, \beta)^{\frac{1}{1 + \epsilon}} + \epsilon_\ell \tag{def.~$M_{1+\epsilon}^*$}
\end{align*}
Now we define the event $\calE := \bigcap_{\ell=1}^\infty \bigcap_{x \in \calA_\ell} \calE_{x,l}(\calA_\ell)$, where 
\begin{align*}
     \calE_{x,l}(\calV) := \left\{  |x^\top \widehat{\theta}_{\ell}(\calV) - x^\top \theta^* | \leq  \epsilon_\ell +   2b \gamma^{1/2} \beta^{-1} M^*_{1 + \epsilon}(\calA, \gamma, \beta)^{\frac{1}{1 + \epsilon}}  \right\},
\end{align*}
with $\widehat{\theta}_{\ell}(\cdot)$ corresponding to $\widehat{\theta}_\ell$ in \pref{alg:MEDE} with an explicit dependence on the action subset. 
Then, we have 
\begin{align*}
    \P\left( \bigcup_{\ell=1}^\infty \bigcup_{x \in \calA_\ell} \{ \mc{E}^c_{x,\ell}( \calA_\ell ) \} \right) &\leq \sum_{\ell=1}^\infty \P\left( \bigcup_{x \in \calA_\ell} \{ \mc{E}^c_{x,\ell}( \calA_\ell) \} \right) \\
&= \sum_{\ell=1}^\infty \sum_{\mc{V} \subseteq \calA} \P\left(\bigcup_{x \in \mc{V}} \{ \mc{E}^c_{x,\ell}( \mc{V} ) \} \,\Big|\, {\calA}_\ell = \mc{V}\right) \P( {\calA}_\ell = \mc{V}) \\
&\leq \sum_{\ell=1}^\infty  \sum_{\mc{V} \subseteq \calA} \tfrac{1}{2\ell^2 T}  \P( {\calA}_\ell = \mc{V} ) \leq \frac{1}{T}, \tag{union bound and $\sum_{\ell = 1}^{\infty} \frac{1}{\ell^2} < 2$}
\end{align*} 
As $\E[R_T \mathbf{1}_{\calE^c}] = \E[R_T|\calE^c] \P[\calE^c] \le (\sup_{x, x'} x'^\top \theta^* - x^\top \theta^*)  T \frac{1}{T} \le 2$, for the rest of the proof we assume event $\calE$.

Let $x^* = \argmax_{x \in \calA} x^\top\theta^*$; then, for every $\ell$ such that $2\epsilon_\ell \ge 4b \gamma^{1/2} \beta^{-1} M_{1 + \epsilon}(\calA, \gamma, \beta)^{\frac{1}{1 + \epsilon}}$ and any $x \in \calA_\ell$, we have
\begin{align*}
    x^\top \widehat{\theta}_\ell - {x^*}^\top \widehat{\theta}_\ell &= (x^\top \widehat{\theta}_\ell - x^\top \theta^*) + (x^\top \theta^* - {x^*}^\top \theta^*) + ({x^*}^\top \theta^*  -{x^*}^\top \widehat{\theta}_\ell) \\
    &\le 2 \epsilon_\ell + 4b \gamma^{1/2} \beta^{-1} M^*_{1 + \epsilon}(\calA, \gamma, \beta)^{\frac{1}{1 + \epsilon}} \tag{def.~$\calE$ and def.~$x^*$} \\&\le 4\epsilon_\ell. \tag{assumption on $\epsilon_\ell$}
\end{align*}
Therefore, recalling the elimination rule in \pref{alg:MEDE}, we have by induction that $x^* \in \calA_{\ell + 1}$.  We also claim that all suboptimal actions of gap more than $8\epsilon_\ell = 16 \epsilon_{\ell+1}$ are eliminated at the end of epoch $\ell$. To see this, let $x' \in \calA_\ell$ be such an action, and observe that
\begin{align*}
    \max_{x \in \calA_\ell} \big( x'^\top \widehat{\theta}_\ell - x^\top \widehat{\theta}_\ell\big) &\ge {x^*}^\top \widehat{\theta}_\ell - x^\top\widehat{\theta}_\ell \tag{$x^* \in \calA_\ell$}
    \\ &\ge {x^*}^\top \theta^* - x^\top \theta^* - 2 \epsilon_\ell - 4b \gamma^{1/2} \beta^{-1} M^*_{1 + \epsilon}(\calA, \gamma, \beta)^{\frac{1}{1 + \epsilon}} \tag{shown above} \\ 
    &\ge {x^*}^\top \theta^* - x^\top \theta^* - 4\epsilon_\ell \tag{assumption on $\epsilon_\ell$}\\ &> 4\epsilon_\ell. \tag{gap exceeds $8\epsilon_\ell$}
\end{align*}
In summary, the above arguments show that when $2\epsilon_\ell \ge 4b \gamma^{1/2} \beta^{-1} M^*_{1 + \epsilon}(\calA, \gamma, \beta)^{\frac{1}{1 + \epsilon}}$, the regret incurred in epoch $\ell+1$ is at most $16 \epsilon_{\ell+1}$.  Since $\calA_{\ell+1} \subseteq \calA_\ell$, this also implies that even when $\ell$ increases beyond such a point, we still incur regret at most $32b \gamma^{1/2} \beta^{-1} M^*_{1 + \epsilon}(\calA, \gamma, \beta)^{\frac{1}{1 + \epsilon}}$.

Finally, we can upper bound the regret as follows:
\begin{align*}
    \E[R_T] &\le \sum_\ell \tau_\ell \Big(\sup_{x \in \calA_\ell} {x^*}^T\theta^* - x^T\theta^*\Big)
    \\ &\le
    \sum_\ell \tau_\ell \max\{16 \epsilon_\ell, 32b \gamma^{1/2} \beta^{-1} M^*_{1 + \epsilon}(\calA, \gamma, \beta)^{\frac{1}{1 + \epsilon}}\} \tag{shown above}
    \\ &\le \sum_\ell 16 \tau_\ell \epsilon_\ell + T \zeta \tag{$\zeta:= 32b \gamma^{1/2} \beta^{-1} M_{1 + \epsilon}(\calA, \gamma, \beta)^{\frac{1}{1 + \epsilon}}$} 
    \\ &\le \sum_{\ell \,:\, 16\epsilon_\ell \ge \omega} 16\epsilon_\ell \tau_\ell + T\omega + T\zeta \tag{for any $\omega \ge 0$} \\
    &\le \sum_{\ell \,:\, 16\epsilon_\ell \ge \omega} 16\epsilon_\ell 32^{\frac{1 + \epsilon}{\epsilon}} (1 + \upsilon)^{\frac{1}{\epsilon}} \varepsilon_\ell^{-\frac{1 + \epsilon}{\epsilon}} M^*_{1 + \epsilon}(\calA, \gamma, \beta)^{\frac{1}{\epsilon}} \log(2l^2|\calA|T) + T (\omega + \zeta) \tag{def.~$\tau_\ell$ in Alg.~\ref{alg:MEDE}}
    \\ 
    &\le \sum_{\ell \,:\, 16\epsilon_\ell \ge \omega} C'_1 (1 + \upsilon)^{\frac{1}{\epsilon}} \varepsilon_\ell^{-\frac{1}{\epsilon}} M^*_{1 + \epsilon}(\calA, \gamma, \beta)^{\frac{1}{\epsilon}} \log(2l^2|\calA|T)  + T (\omega + \zeta) \tag{for some constant $C'_1$} \\
    &\le C_1(1 + \upsilon)^{\frac{1}{1 + \epsilon}} M^*_{1 + \epsilon}(\calA, \gamma, \beta)^{\frac{1}{1 + \epsilon}} \log(2|\calA|T\log_2^2 T)^{\frac{\epsilon}{1 + \epsilon}} T^\frac{1}{1 + \epsilon}  + T \zeta \tag{$\omega := M^*_{1 + \epsilon}(\cdot)^{\frac{1}{1 + \epsilon}} \log(2 |\calA| T\log_2^2 T)^{\frac{\epsilon}{1 + \epsilon}} T^{\frac{-\epsilon}{1 + \epsilon}}$ and $\ell \le \log_2 T$; see below}
    \\ &\le \left( C_0\beta^{-1} b + C_1(1 + \upsilon)^{\frac{1}{1 + \epsilon}} \log(|\calA| T \log_2^2 T)^{\frac{\epsilon}{1 + \epsilon}} \right) M^*_{1 + \epsilon}(\calA, T^{\frac{-2\epsilon}{1 + \epsilon}}, \beta)^{\frac{1}{1 + \epsilon}} T^{\frac{1}{1 + \epsilon}}. \tag{def.~$\zeta$ and $\gamma = T^{\frac{-2\epsilon}{1 + \epsilon}}$}
\end{align*}
In more detail, the second-last step upper bounds $\sum_{\ell \,:\, 16\epsilon_\ell \ge \omega} \epsilon_\ell^{-\frac{1}{\epsilon}}$ by a constant times its largest possible term $\omega^{-\frac{1}{\epsilon}}$, since $\{\epsilon_\ell\}_{\ell \ge 1}$ is  exponentially decreasing.  Since the choice of $\omega$ contains $(M^*_{1+\epsilon})^{\frac{1}{1+\epsilon}}$, the overall $M^*_{1+\epsilon}$ dependence simplifies as $\big( \frac{M^*_{1+\epsilon}}{(M^*_{1+\epsilon})^{\frac{1}{1+\epsilon}}} \big)^{\frac{1}{\epsilon}} = \big((M^*_{1+\epsilon})^{\frac{\epsilon}{1+\epsilon}} \big)^{\frac{1}{\epsilon}} = (M^*_{1+\epsilon})^{\frac{1}{1+\epsilon}}$.

\section{Unit Ball Lower Bound}
\label{app: lower-unit-ball}

In this appendix, we prove the following lower bound for the case that the action set is the unit ball.

\begin{theorem}
\label{thm: Lower-unit-ball}
Let the action set be $\mathcal{A} = \{x \in \mathbb{R}^{d} \,:\, \|x\|_2 \le 1\} $, and the $(1 + \epsilon)$-absolute moment of the error distribution be bounded by $1$.  Then, for any algorithm, there exists $\theta^* \in \mathbb{R}^{d}$ such that  $\sup_{x \in \calA} |x^\top \theta^*| \le 1$, and such that for $T \ge d^2$, the regret incurred is $\Omega( d^{ \frac{2\epsilon}{1 + \epsilon}} T^{\frac{1}{1 + \epsilon}})$.
\end{theorem}

    Since the KL divergence between Bernoulli random variables Ber$(p)$ and Ber$(q)$ goes to infinity as $p \rightarrow 0$, and $\theta^\top x$ can be zero for unit ball, we cannot use the same reward distribution as before. However, we can overcome this by shifting all probabilities and adding $-1$ to the support of the reward random variable.
    Specifically, we set the error distribution to be: 
    \begin{align*}
        y(x) = \begin{cases}
            (\frac{1}{\gamma})^{\frac{1}{\epsilon}} & w.p.~\,\gamma^{\frac{1}{\epsilon}} (\theta^\top x + 2\sqrt{d}\Delta) \\
            0 & w.p.~\,1 - \gamma^{\frac{1}{\epsilon}}(\theta^\top x + 2\sqrt{d}\Delta) - 2\sqrt{d}\Delta \\ 
            -1 & w.p.~\,2\sqrt{d}\Delta 
        \end{cases}
    \end{align*}
    with $\gamma := 24\sqrt{d} \Delta$ and $\Delta$ to be specified later. For any $\theta \in \{ \pm \Delta\}^d$, the absolute value of rewards are bounded by $\sum_{i = 1}^d \frac{1}{\sqrt{d}} \Delta = \sqrt{d}\Delta$. Then, assuming $\Delta \le \frac{1}{24\sqrt{d}}$, we have $|\theta^\top x| \le \sqrt{d} \Delta \le \frac{1}{8}$ and $\|\theta\|_2 \le 1$ as well as $\gamma \le 1$, and the $(1 + \epsilon)$-central absolute moment is bounded by:
    \begin{align*}
        & \E[|y(x) - \theta^\top x|^{1 + \epsilon} \;|\, x] \\ 
        &\le |\gamma ^{-\frac{1}{\epsilon}} - \theta^\top x
        |^{1 + \epsilon} (\theta^\top x + 2\sqrt{d} \Delta ) + |\theta^\top x|^{1 + \epsilon} + |-1-\theta^\top x|^{1 + \epsilon}2\sqrt{d}\Delta \tag{$\gamma \le 1$}
        \\ &\le 2^{1 + \epsilon} \gamma^{-1} 3 \sqrt{d} \Delta  + (\sqrt{d} \Delta)^{1 + \epsilon} + 2\sqrt{d}\Delta (\sqrt{d}\Delta + 1)^{1 + \epsilon} \tag{ $|\theta^\top x| \le \sqrt{d}\Delta \le 1$ and $\gamma^{-\frac{1}{\epsilon}} \ge 1$} \\ 
        &\le \frac{2^{1 + \epsilon}}{8} + \left(\frac{1}{24}\right)^{1 + \epsilon} + \frac{1}{12} \left(\frac{9}{24}\right)^{1 + \epsilon} < 1 \tag{def. $\gamma$, $\Delta \le \frac{1}{24\sqrt{d}}$, and $\epsilon \in (0,1]$}.
    \end{align*}
     Defining $T_i := T \wedge \min(s: \sum_{t = 1}^s x_{t,i}^2 \ge \frac{T}{d})$, we have 
    \begin{align*}
                R_T(\calA,\theta) &= \Delta \E_{\theta} \left[ \sum_{t = 1}^T \sum_{i = 1}^d \left(\frac{1}{\sqrt{d}} - x_{t, i} \sign(\theta_i) \right) \right]
\\ & \ge
        \frac{\Delta \sqrt{d}}{2} \E_{\theta} \left[ \sum_{t = 1}^T  \sum_{i = 1}^d \left(\frac{1}{\sqrt{d}} - x_{t, i} \sign(\theta_i) \right)^2 \right] \tag{by expanding the square and applying $\|x_{t}\|^2_2 \le 1$} \\
        &\ge \frac{\Delta \sqrt{d}}{2} \sum_{i = 1}^d \E_{\theta}\left[\sum_{t = 1}^{T_i} \left(\frac{1}{\sqrt{d}} - x_{t, i} \sign(\theta_i) \right)^2\right]. 
    \end{align*}

    Now we define $U_i(b) := \sum_{t = 1}^{T_i} \big(\frac{1}{\sqrt{d}} - x_{t, i}b \big)^2 $, which gives 
    \begin{align*}
       U_i(1) \le 2 \sum_{t=1}^{T_i} \frac{1}{d} + 2 \sum_{t = 1}^{T_i} x^2_{t, i} \le \frac{4T}{d} + 2.
    \end{align*}
    Then, for any $\theta, \theta' \in \{\pm \Delta\}^d$ that only differ in $i$-th element, we have 
    \begin{align*}
        \E_\theta[U_i(1)] &\ge \E_{\theta'}[U_i(1)] - \left( \frac{4T}{d} + 2 \right) \sqrt{\frac{1}{2} \KL(\P_\theta\|\P_{\theta'})} \tag{Pinsker's inequality} \\
        & \ge \E_{\theta'}[U_i(1)] - \left( \frac{4T}{d} + 2 \right) \sqrt{\frac{1}{2} \E_\theta\left[ \sum_{t = 1}^{T_i} \KL(y_\theta(x_t)\| y_{\theta'}(x_t)) \right]} \tag{Chain rule} \\
        & \ge
        \E_{\theta'}[U_i(1)] - \left( \frac{4T}{d} + 2 \right) \sqrt{\frac{1}{2} \E_\theta\left[ \sum_{t = 1}^{T_i} 24^{\frac{1}{\epsilon}} 8 \sqrt{d}^{\frac{1 - \epsilon}{\epsilon}} \Delta^{\frac{1 + \epsilon}{\epsilon} } x^2_{t, i} \right]} \tag{Inverse Pinsker's inequality; see below}
        \\ & \ge
        \E_{\theta'}[U_i(1)] - 24^{\frac{1}{2\epsilon}} 2 \Delta^{\frac{1 + \epsilon}{2\epsilon}} \sqrt{d}^{\frac{1 - \epsilon}{2\epsilon}} \left( \frac{4T}{d} + 2 \right) \sqrt{ \E_\theta\left[ \sum_{t = 1}^{T_i} x_{t,i}^2 \right]}
        \\ &\ge
        \E_{\theta'}[U_i(1)] - 24^{\frac{1}{2\epsilon}} 12\sqrt{2} \Delta^{\frac{1 + \epsilon}{2\epsilon}} \sqrt{d}^{\frac{1 - \epsilon}{2\epsilon}} \frac{T}{d} \sqrt{\frac{T}{d}}. \tag{$d \le T$, $\sum_{t = 1}^{T_i} x_{t,i}^2 \le \frac{T}{d}+1$}
    \end{align*}
    Note that the version of the chain rule with a random stopping time can be found in \cite[Exercise 15.7]{Csa18}.  We detail the step using inverse Pinsker's inequality (\cite{DBLP:journals/corr/Sason15b}) as follows:
    \begin{align*}
        \KL(y_\theta(x_t)\| y_{\theta'}(x_t)) &\le \frac{2}{\min_{a\in\{\gamma^{- \frac{1}{\epsilon}}, 0, -1\}} \P[y_{\theta'}(x_t) = a]} \sup_a\left| \P[y_{\theta}(x_t) = a] - \P[y_{\theta'}(x_t) = a]\right|^2 \\
        & \le \frac{2}{\gamma^{\frac{1}{\epsilon}}\sqrt{d}\Delta} (\gamma^{\frac{1}{\epsilon}} 2 \Delta x_{t, i} )^2 
        \\ 
        & \le 24^{\frac{1}{\epsilon}} 8  \sqrt{d}^{\frac{1}{\epsilon} - 1} \Delta^{\frac{1}{\epsilon} + 1} x^2_{t, i}. \tag{$\gamma = 24\sqrt{d}\Delta$}
    \end{align*}
    Using the above lower bound on $\E_\theta[U_i(1)]$, and setting $\Delta := 24^{\frac{-1}{1 + \epsilon}} d^{\frac{3 \epsilon - 1}{2(1 + \epsilon)}} \left(288 T\right)^{\frac{-\epsilon}{1 + \epsilon}}$ (noting $288 = (12\sqrt{2})^2$), we have the following: 
    \begin{align*}
        \E_{\theta}[U_i(1)] + \E_{\theta'}[U_i(-1)] &\ge \E_{\theta'}[U_i(1) + U_i(-1)] -  24^{\frac{1}{2\epsilon}} 12\sqrt{2} \Delta^{\frac{1 + \epsilon}{2\epsilon}} \sqrt{d}^{\frac{1 - \epsilon}{2\epsilon}} \frac{T}{d} \sqrt{\frac{T}{d}}
        \\ & = 2\E_{\theta'}\left[ \frac{T_i}{d} + \sum_{t = 1}^{T_i} x_{t, i}^2 \right] - 24^{\frac{1}{2\epsilon}} 12\sqrt{2} \Delta^{\frac{1 + \epsilon}{2\epsilon}} \sqrt{d}^{\frac{1 - \epsilon}{2\epsilon}} \frac{T}{d} \sqrt{\frac{T}{d}}
        \\ &\ge \frac{2T}{d} - \frac{T}{d} = \frac{T}{d}. \tag{$T_i \ge 0$, def.~$T_i$, choice of $\Delta$}
    \end{align*}
    Note also that $\Delta \le \frac{1}{24\sqrt{d}}$ (as required earlier) since $T \ge d^2$. We now combine the preceding equation with our earlier lower bound on $R_T$.  By averaging overall $\theta \in \{\pm \Delta\}^d$, we conclude that there exists some $\theta^*$ such that
    \begin{align*}
        R_T(\calA, \theta^*) &\ge \frac{\Delta \sqrt{d}}{2} \frac{1}{2^d} \sum_{\theta \in \{-\Delta,\Delta\}^d} R_T(\calA, \theta) \\ &\ge \frac{\Delta \sqrt{d}}{4} \sum_{i = 1}^d \sum_{\theta_i \in \{-\Delta,\Delta\}} \E_{\theta}[U_i(\sign(\theta_i))] \tag{$R_T$ bound and $\sum_{\{\theta_j\}_{j \ne i}} 1 = 2^{d-1}$}.
        \\ &\ge \frac{1}{4} T \sqrt{d} \Delta \tag{$\E_{\theta}[U_i(1)] + \E_{\theta'}[U_i(-1)] \ge \frac{T}{d}$} \\ 
        &\ge \frac{1}{4 \cdot 24 \cdot 12\sqrt{2}} d^{\frac{2\epsilon}{1 + \epsilon}}  {T}^{\frac{1}{1 + \epsilon}}. \tag{choice of $\Delta$, $\epsilon \in [0,1]$}
    \end{align*}

\section{Extension to Kernel Bandits} \label{app:kernel}

\subsection{Problem Setup}
We consider an unknown reward function $f:\calA \rightarrow \bbR$ lying in the reproducing kernel Hilbert space (RKHS) $\mathcal{H}$ associated with a given kernel $K$, i.e., $f(x) = \langle f, K(x,\cdot) \rangle_{K}$. Similar to the linear bandit setting, we assume that $\max_{x \in \calA} |f(x)| \le 1$ and $\|f\|_K \le b$ for some $b > 0$. 

At each round $t=1,2,\dots,T$,
the learner chooses an action $x_t\in\calA \subseteq [0, 1]^d$ and observes the reward
\[
    y_t \;=\; f(x_t) \;+\; \eta_t,
\]
where $\eta_t$ are independent noise terms that satisfy $\mathbb{E}[\eta_t]=0$ and
$\mathbb{E}\bigl[|\eta_t|^{1+\epsilon}\bigr]\le\upsilon$ for some $\epsilon\in(0,1]$ and
finite $\upsilon>0$. 
Letting $x^\star \in \arg\max_{x\in[0, 1]^d}f(x)$ be an optimal
action,
the cumulative expected regret after $T$ rounds is
\[
    R_T \;=\; \sum_{t=1}^{T} \big( f(x^*) - f(x_t) \big).
        \]
Given $(\mathcal{A},\epsilon,\upsilon)$, the objective is to design a policy for sequentially selecting the points (i.e., $x_t$ for $t=1,\dotsc,T$) in order to minimize $R_T$. We focus on the Mat\'ern kernel, defined as follows:
\[K_{\nu, l}(x, x') := \frac{2^{1 - \nu}}{\Gamma(\nu)} \left(\frac{\|x - x'\|_2 \sqrt{2 \nu}}{l}\right)^{\nu} B_\nu\left(\frac{\|x - x'\|_2 \sqrt{2 \nu}}{l}\right),
\]
where $\Gamma$ is the Gamma function, $B_\nu$ is the modified Bessel function, and $(\nu,l)$ are parameters corresponding to smoothness and lengthscale.

We focus on the case that $\calA$ is a \emph{finite} subset of $[0,1]^d$, but it is well known (e.g., see \cite[Assumption 4]{Vak21}) that the resulting regret bounds extend to the continuous domain $[0,1]^d$ via a discretization argument with with $\log|\calA| = O(\log T)$.

\subsection{Proof of \pref{cor:Matern}}

We state a more precise version of \pref{cor:Matern} as follows.

\begin{theorem} \label{thm:matern_full}
    For any unknown reward function $f:\calA \rightarrow \bbR$ lying in the RKHS of the Matérn kernel with parameters $(\nu, l)$, for some finite set $\calA \subseteq [0, 1]^d$, assuming that $\max_{x \in \calA} |f(x)| \le 1$ and $\|f\|_K \le b$ for some $b > 0$, we have
    \begin{align*}
        M^*(\calA, T^{\frac{-2\epsilon}{1 + \epsilon}}, 1) \le C  T^{\epsilon \cdot \frac{d}{2\nu + d}}, 
    \end{align*}
    for some constant $C$, and \pref{alg:MEDE} achieves regret of
    \begin{align*}
            R_T(f,\calA) \le \left( C'_0 b + C'_1(1 + \upsilon)^{\frac{1}{1 + \epsilon}} \log(|\calA| T \log^2 T)^{\frac{\epsilon}{1 + \epsilon}} \right)  T^{1 - \frac{\epsilon}{1 + \epsilon}\frac{2\nu}{2\nu + d}}, 
    \end{align*}
    for some constants $C'_0, C'_1$. Note that the constants may depend on the kernel parameters $(\nu, l)$ and the dimension $d$. 
\end{theorem}

We now proceed with the proof.  
We first argue that \pref{alg:MEDE} and \pref{thm: main-upper} can still be applied (with $x$ replacing $a$ and $f(x)$ replacing $a^\top \theta^*$) in the kernel setting.  The reasoning is the same as the case $\epsilon = 1$ handled in \cite{Cam21}, so we keep the details brief.

Recall that for any kernel $K$, there exists a (possibly infinite dimensional) feature map $\phi: \calA \rightarrow \mathcal{H}$ such that $K(x, x') = \phi(x)^\top \phi(x')$.  For any $\lambda \in \Delta_{\calA}$, we define $k_\lambda(\cdot) \in \bbR^{|\calA|}$ such that for $\psi \in \mathcal{H}$, $k_\lambda(\psi)_i := \sqrt{\lambda_i} \phi(x_i)^\top \psi$, and $K_\lambda \in \bbR^{|\calA| \times |\calA|}$ such that $(K_\lambda)_{i,j} := \sqrt{\lambda_i} \sqrt{\lambda_j} K(x_i, x_j)$.  Then similar to \cite[Lemma 2]{Cam21}, we have for any $\psi, \rho \in \mathcal{H}$ that
\begin{align*}
    \psi^\top A^{(\gamma)}(\lambda)^{-1} \rho = \gamma^{-1} \psi^\top \rho - \gamma^{-1} k_\lambda(\psi) (K_\lambda + I_{|\calA|})^{-1} k_\lambda(\rho).
\end{align*}
Then the gradient for the experimental design problem $\inf_{\lambda \in \Delta_{\calV}} \max_{v \in \calV} \|\phi(v)\|_{A^{(\gamma)} (\lambda)^{-1} }$ (which is an upper bound for our experimental design objective $M_{1 + \epsilon}(\lambda; \calV, \gamma, 1)$ by the proof of \pref{lem:M_bound}) can be computed efficiently.  Moreover, \pref{thm: main-upper} still holds because the the kernel setup can be viewed as a linear setup in an infinite-dimensional feature space (after applying the feature map $\phi$ to the action set), and our analysis does not use the finiteness of the dimension. 

Given \pref{thm: main-upper}, the main remaining step is to upper bound $M_{1+\epsilon}^*$.   To do so, we use the well-known polynomial eigenvalue decay of the Mat\'ern kernel.  Specifically, the $j$-th eigenvalue $\varphi_j$ satisfies $\varphi_j \le \order(j^{-\kappa})$ with $\kappa = \frac{2\nu + d}{d}$ (e.g., see \cite{Vak21}). 
We let $\lambda^*_{D} \in \arg\max_{\lambda \in \Delta_{\calA}} \log \det\left(A^{(\gamma)}(\lambda)\right)$, and proceed as follows:
 \begin{align*}
     M^*_{1 + \epsilon}(\calA, \gamma, 1)^{\frac{2}{1 + \epsilon}} &\le  \max_{\calV \in \calA} \inf_{\lambda \in \Delta_{\calV}} 2^{\frac{2}{1 + \epsilon}}\max_{v \in \calV} \|\phi(v)\|^2_{A^{(\gamma)} (\lambda)^{-1} } \tag{shown in the proof of \pref{lem:M_bound}} \\ 
     &\leq 4\tr\left(A(\lambda_D^*)(A(\lambda_D^*) + \gamma I)^{-1}\right) \tag*{ \cite[Lemma 3]{Cam21} }\\
    &= 4\tr\left(K_{\lambda_D^*} (K_{\lambda_D^*} + \gamma I)^{-1}\right) \\ 
    &= 4\sum_{j = 1}^{|\calA|} \frac{\varphi_j}{\varphi_j + \gamma} \\
    &\le 4\sum_{j=1}^{|\calA|} \frac{cj^{-\kappa}}{cj^{-\kappa} + \gamma} \tag{for some constant $c \ge 1$  dependent on $l, \nu, d$}\\
    &\le 4c \sum_{j \leq \gamma^{-\frac{1}{\kappa}}} \frac{j^{-\kappa}}{j^{-\kappa} + \gamma} + 4c \sum_{j > \gamma^{-\frac{1}{\kappa}}} \frac{j^{-\kappa}}{j^{-\kappa} + \gamma} \tag{$c \ge 1$} \\
    &\leq 4c\gamma^{-1/\kappa} + 4c\sum_{j > \gamma^{-\frac{1}{\kappa}}} \frac{j^{-\kappa}}{\gamma} \tag{dropping terms in denominators}  \\
    &\leq 4c\gamma^{-\frac{1}{\kappa}} + 4c(\gamma^{-\frac{1}{\kappa}})^{1 - \kappa} \frac{1}{(\kappa - 1)\gamma} \tag{bounding sum by integral; $\kappa > 1$} \\
    & = 4c\gamma^{-\frac{1}{\kappa}}\left(1 + \frac{1}{\kappa - 1}\right) \\
    & = 4 c \frac{2\nu + d}{2\nu} T^{\frac{2 \epsilon}{1 + \epsilon} \frac{d}{2 \nu + d}} \tag{$\gamma = T^{\frac{- 2 \epsilon}{1 + \epsilon}}$ and $\kappa = \frac{2\nu + d}{d}$}.
 \end{align*}
Taking the square root on both sides gives $M^*_{1 + \epsilon}(\calA, \gamma, 1)^{\frac{1}{1 + \epsilon}} = \otil\big( T^{\frac{\epsilon}{1 + \epsilon} \frac{d}{2 \nu + d}}  \big)$, and multiplying by $\otil(T^{\frac{1}{1+\epsilon}}) = \otil(T^{1 - \frac{\epsilon}{1+\epsilon}})$ from the regret bound in \pref{thm: main-upper} gives $\otil(T^{1 - \frac{\epsilon}{1+\epsilon} \cdot \frac{ 2 \nu}{ 2 \nu + d}})$ regret as claimed in \pref{cor:Matern}.  By the same reasoning but keeping track of the logarithmic terms, we obtain the regret bound stated in \pref{thm:matern_full}.

\subsection{Comparisons of Bounds} \label{sec:kernel_cmp}
{\bf Comparison to existing lower bound.} In \pref{fig:Matern}, we compare our regret upper bound to the lower bound of $\Omega\big( T^{\frac{\nu+d\epsilon}{\nu(1+\epsilon)+d\epsilon}} \big)$ proved in \cite{Cho19}.  We see that the upper and lower bounds coincide in certain limits and extreme cases:
\begin{itemize}[leftmargin=5ex]
    \item As $\nu/d \to \infty$, the regret approaches $T^{\frac{1}{1+\epsilon}}$ scaling, which matches the regret of linear heavy-tailed bandits in constant dimension.
    \item As $\nu/d \to 0$ and/or $\epsilon \to 0$, the regret approaches trivial linear scaling in $T$.
    \item When $\epsilon = 1$, the regret scales as $\widetilde{\Theta}\big( T^{\frac{\nu + d}{2\nu + d}} \big)$, which matches the optimal scaling for the sub-Gaussian noise setting \cite{Sca17a}.  As we discussed earlier, this finite-variance setting was already handled in \cite{Cam21}.
\end{itemize}
For finite $\nu/d$ and fixed $\epsilon \in (0,1)$, we observe from \pref{fig:Matern} that gaps still remain between the upper and lower bounds, but they are typically small, especially when $\nu / d$ is not too small.

{\bf Comparison to existing upper bound.} In \cite{Cho19}, a regret upper bound of $\otil(\gamma_T T^{\frac{2+\epsilon}{2(1+\epsilon)}})$ was established, where $\gamma_T$ is an \emph{information gain} term that satisfies $\gamma_T = \otil(T^{\frac{d}{2\nu+d}})$ for the Mat\'ern kernel \cite{Vak20a}.  We did not plot this upper bound in \pref{fig:Matern}, because its high degree of suboptimality is easier to describe textually:
\begin{itemize}[leftmargin=5ex]
    \item For $\nu/d = 1/4$ and $\nu/d = 1$, their bound exceeds the trivial $\order(T)$ bound for all $\epsilon \in (0,1]$.
    \item For $\nu/d = 4$, their bound still exceeds $\order(T)$ for $\epsilon \lesssim 0.28$, and is highly suboptimal for larger $\epsilon$.
    \item As $\nu/d \to \infty$, the $\gamma_T$ term becomes insignificant and their bound simplifies to $\otil(T^{\frac{2+\epsilon}{2(1+\epsilon)}})$, which is never better than $\otil(T^{3/4})$ (achieved when $\epsilon = 1$).  
    \item A further weakness when $\epsilon = 1$ is that the optimal $\gamma_T$ dependence should be $\sqrt{\gamma_T}$ rather than linear in $\gamma_T$ \cite{Sca17a,Cam21}.
\end{itemize}
For the \emph{squared exponential kernel}, which has exponentially decaying eigenvalues rather than polynomial, these weaknesses were overcome in \cite{Cho19} using kernel approximation techniques, to obtain an optimal $\otil(T^{\frac{1}{1+\epsilon}})$ regret bound.  Our main contribution above is to establish a new state of the art for the Mat\'ern kernel, which is significantly more versatile in being able to model both highly smooth (high $\nu$) and less smooth (small $\nu$) functions.

\begin{figure}[!tbp]
  \centering
  \includegraphics[width=0.5\linewidth]{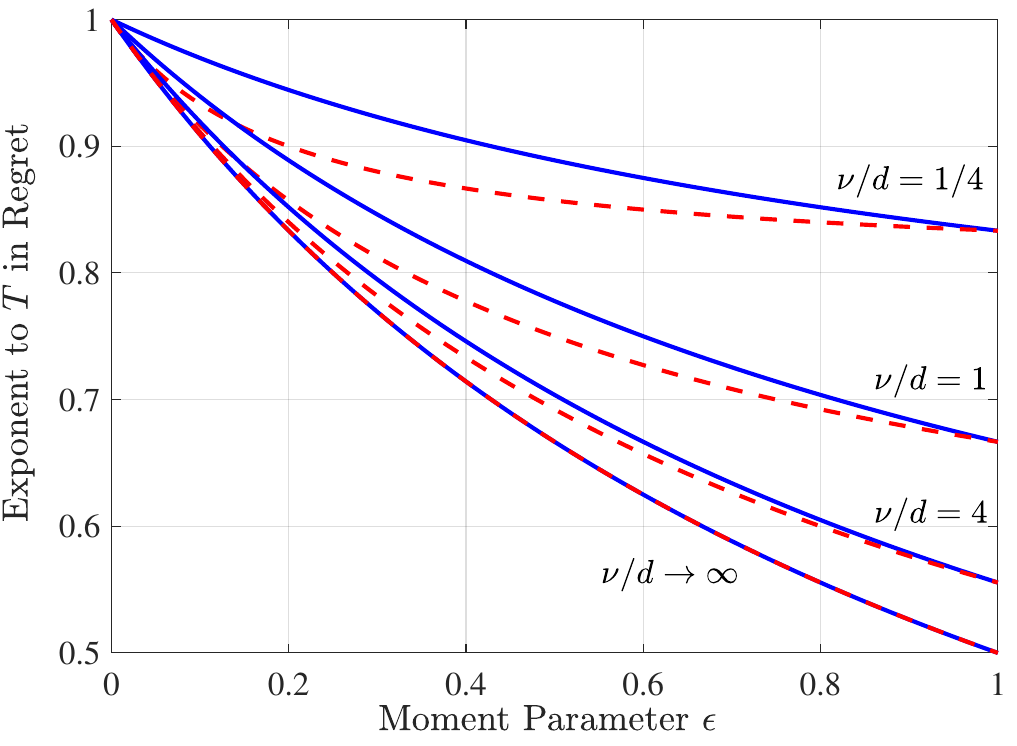}
    \caption{Comparison of our regret upper bound (solid) and the lower bound of \cite{Cho19} (dashed).  We plot the exponent $c$ such that the regret bound has dependence $T^{c}$, with the 4 pairs of curves corresponding to $\nu/d \in \{0.25,1,4\}$ and $\nu/d \to \infty$.}
  \label{fig:Matern}
\end{figure}

\section{Numerical Experiments}
\label{app: numer-exp}

In this section, we perform a simple proof-of-concept experiment to demonstrate that our algorithm can outperform existing methods as the ambient dimension increases.  However, we emphasize that our main contributions are theoretical, and we leave detailed experimental studies for future work.

We conduct experiments with horizon $T=100{,}000$ and action set size $N=2d$. The true parameter is $\theta^\star=\tfrac{1}{\sqrt{d}}\mathbf{1}$ (so $\|\theta^\star\|_2=1$), and the action set is the subset of the normalized hypercube given by the signed coordinate directions $\mathcal{A}=\{\pm e_i\}_{i=1}^d$. Rewards follow $r_t=x_t^\top \theta^\star+\eta_t$ with heavy-tailed noise $\eta_t \sim \text{ParetoII}(\alpha=2,\sigma=1)-\mathbb{E}[\text{ParetoII}(\alpha=2,\sigma=1)]$ (centered to zero mean). The proposed algorithm is instantiated with input $\epsilon=0.5$ and evaluated against the \emph{Confidence Region with Truncated Mean} (CRTM) Algorithm in ~\cite{XueWangWanYiZhang2023}. Performance is measured via cumulative pseudo-regret $\sum_{t=1}^T\big(x_t^{\star\top}\theta^\star-x_t^\top\theta^\star\big)$, aggregated over $10$ independent repetitions (with identical arm sets and independent noise).

\paragraph{Results.}
As shown in \pref{fig:placeholder}, for all $d \ge 40$ \pref{alg:MEDE} achieves comparable or lower mean regret than CRTM, and notably, the gap widens as $d$ increases. Both procedures remain sublinear in $T$ in this controlled setting; however, the regret of \pref{alg:MEDE} grows more slowly with $d$, consistent with the guarantees in \pref{thm: main-upper}.

\begin{figure}
    \centering
    \includegraphics[width=0.75\linewidth]{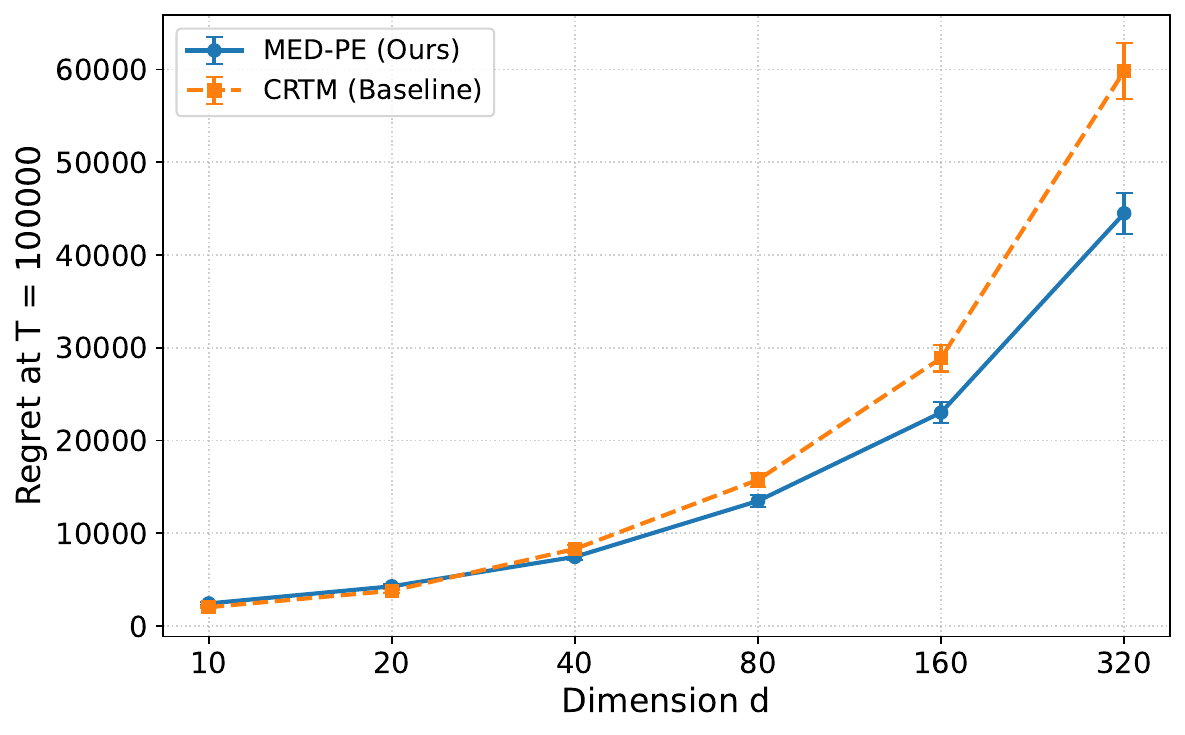}
    \caption{Regret vs dimension $d$ with time horizon $T = 100,000$, and $N = 2d$ arms}
    \label{fig:placeholder}
\end{figure}

\section{Proof of \pref{thm: lower-finite} (Finite-Arm Lower Bound)}
\label{app: finite-lower}
Consider $\log(\cdot)$ with base 2, 
and define $m$ to be the smallest integer such that $\frac{m}{\log m} \ge \frac{d}{\log n}$.  From the assumption $n \in [d, 2^{\lfloor\frac{d}{4}\rfloor}]$ we can readily verify that $d > 4$ and $m \in [4,d]$. For convenience, we assume that $d$ is a multiple of $m$, since otherwise we can form the construction of the lower bound with $d' = d - (d \text{ mod } m)$ and pad the action vectors with zeros. Letting $d_i := (i - 1) m$, we define the action set and the parameter set as follows for some $\Delta$ to be specified later: 
\begin{gather*}
    \mathcal{A} := \bigg\{a \in \{0, 1\}^{d}: \sum_{j = d_i + 1}^{d_{i + 1}} a_{j} = 1, ~~\forall i \in [d/m] \bigg\} \\
    \theta^* \in \Theta := \left\{ \theta \in \{\Delta, 2\Delta\}^{d} : \sum_{j = d_i + 1}^{d_{i + 1}} \theta_j = (m + 1)\Delta, ~~\forall i \in [d/m]  \right\}.
\end{gather*}    
In simple terms, the $d$-dimensional vectors are arranged in $d/m$ groups of size $m$; each block in $a \in \calA$ has a single entry of 1 (with 0 elsewhere), and each block in $\theta^*$ has a single entry of $2\Delta$ (with $\Delta$ elsewhere). Observe that if $\Delta \le \min(\frac{m}{4d}, \frac{1}{4\sqrt{d}} )$, then $\|\theta^*\|_2 \le 1$ and $x^\top \theta^* \le 1$ as required.  
Moreover, we have $|\calA| = m^{\frac{d}{m}}$, and thus $\log |\calA| = \frac{d}{m} \log m \le \log n$ by the definition of $m$.

Similar to \pref{thm: lower-hypercube}, we let the reward distribution be
    $$
        y(x) = \begin{cases}
            (\frac{1}{\gamma})^{\frac{1}{\epsilon}} & \text{w.p.~\,}\gamma^{\frac{1}{\epsilon}} \theta^\top x \\
            0 & \text{w.p.~\,}1 - \gamma^{\frac{1}{\epsilon}} \theta^\top x
        \end{cases} $$
with $\gamma := 2 \Delta \frac{d}{m}$.  The choices of $\calA$ and $\Theta$ give $\theta^\top x \le 2\Delta \frac{d}{m}$, so by the same reasoning as in \pref{thm: lower-hypercube}, the $(1 + \epsilon)$-moment of the reward distribution is bounded by $1$.

Let $\text{ind}_i(x) := \arg\max_{b \in [m]} (x_{d_i + b})$ for fixed $x \in \calA \cup \Theta$, and define $T_{i,b} := |\{t : x_{t, d_i + b} = 1\}|$. Moreover, define $t_{\rm U}$ to be a random integer drawn uniformly from $[T]$, which immediately implies that $\P_\theta[x_{t_{\rm U}, d_i + b} = 1] = \frac{\E_\theta[T_{i, b}]}{T}$.  Then,
    \begin{align*}
        R_T(\calA,\theta) &= \sum_{t = 1}^T \sum_{i = 1}^{d/m} \big( \Delta - \Delta \mathbb{I}\{\text{ind}_i(\theta) = \text{ind}_i(x_t)\} \big)
        \\ &=
            \Delta \sum_{i = 1}^{d/m} \big( T - \E_{\theta}\big[ T_{i, \text{ind}_i(\theta)} \big] \big) 
        \\ &=
        \Delta T \sum_{i = 1}^{d/m} \big( 1 - \P_\theta[x_{t_{\rm U}, d_i + \text{ind}_i(\theta)} = 1] \big).
    \end{align*}
    For fixed $\theta \in \Theta$ and $i \in [\frac{d}{m}]$, and any $b \in [m]$, we define $\theta^{(b)} \in \Theta$ to have entries given by $\theta^{(b)}_{j} = \begin{cases}
        \Delta + \Delta\mathbb{I}\{j = d_i + b\} & j \in [d_i + 1, d_{i + 1}] \\ \theta_j &\text{otherwise}
    \end{cases}$; and define the base parameter $\theta^{(0)}$ with entries $\theta^{(0)}_{j} = \begin{cases}
        \Delta & j \in [d_i + 1, d_{i + 1}] \\ \theta_j &\text{otherwise}
    \end{cases}$. Note that $\theta^{(\text{ind}_i(\theta))} = \theta$, and that the dependence of $\theta^{(b)}$ on $i$ is left implicit.
    
    Then, for $b \in [m]$, we have
\begin{align*}
    \P_{\theta^{(b)}}[x_{t, d_i + b} = 1] &\le \P_{\theta^{(0)}}[x_{t, d_i + b} = 1] + \sqrt{\frac{1}{2} \KL(\P_{\theta^{(0)}} \| \P_{\theta^{(b)}}) } \tag{Pinsker's Inequality}
    \\ &= \P_{\theta^{(0)}}[x_{t, d_i + b} = 1] + \sqrt{ \frac{1}{2} \E_{\theta^{(0)}} \left[ \sum_{t = 1}^T \KL\left( \text{Ber}(\gamma^{\frac{1}{\epsilon}} {\theta^{(0)}}^\top x_t) \| \text{Ber}(\gamma^{\frac{1}{\epsilon}} {\theta^{(b)}}^\top x_t)\right) \right] }. \tag{Chain rule}
\end{align*}
Similarly to the proof of \pref{thm: lower-hypercube}, applying $\KL(\text{Ber}(p) \| \text{Ber}(q)) \le \frac{(p - q)^2}{q(1 - q)}$ along with $\Delta d/m \le \theta^\top x \le 2\Delta d/m$ and $|(\theta^{(0)}-\theta^{(b)})^\top x| \le \Delta$ gives
    \begin{align*}
        &\KL\left( \text{Ber}(\gamma^{\frac{1}{\epsilon}} {\theta^{(0)}}^\top x_t) \| \text{Ber}(\gamma^{\frac{1}{\epsilon}} {\theta^{(b)}}^\top x_t)\right) \le \frac{2 (\gamma^{\frac{1}{\epsilon}} (\theta^{(0)} - \theta^{(b)})^\top x_t)^2}{\gamma^{\frac{1}{\epsilon}} {\theta^{(b)}}^\top x_t}
        \\ & \qquad \le \frac{2^{\frac{2 + \epsilon}{\epsilon}}\Delta^{\frac{2}{\epsilon}} (\frac{d}{m})^{\frac{2}{\epsilon}} \Delta^2 \mathbb{I}\{x_{t,d_i + b} = 1\}}{2^{\frac{1}{\epsilon}}\Delta^{\frac{1 + \epsilon}{\epsilon}} (\frac{d}{m})^{\frac{1 + \epsilon}{\epsilon}}} = 2^{\frac{1 + \epsilon}{\epsilon}} \Delta^{\frac{1 + \epsilon}{\epsilon}} \left(\frac{d}{m}\right)^{\frac{1 - \epsilon}{\epsilon}} \mathbb{I}\{x_{t,d_i + b} = 1\}.
    \end{align*}
    We set $\Delta := \frac{1}{8}\left(\frac{d}{m}\right)^{\frac{\epsilon - 1}{1 + \epsilon}} \left(\frac{T}{m}\right)^{\frac{-\epsilon}{1 + \epsilon}}$.  We claim that under this choice, the condition  $T \ge 4^{\frac{1 + \epsilon}{\epsilon}}d^{\frac{1 + \epsilon}{\epsilon}}$ implies $\Delta \le \min(\frac{m}{4d}, \frac{1}{4\sqrt{d}})$, as we required earlier.  To see this, we rewrite $\Delta = \frac{1}{8} d^{\frac{\epsilon-1}{1+\epsilon}} m^{\frac{1}{1+\epsilon}} T^{-\frac{\epsilon}{1+\epsilon}}$ and substitute the bound on $T$ to obtain $\Delta \le \frac{1}{32} d^{\frac{\epsilon-1}{1+\epsilon}} m^{\frac{1}{1+\epsilon}} d^{-1}$.  Dividing both sides by $m$ gives $\frac{\Delta}{m} \le \frac{1}{32d}$, whereas applying $m \le d$ gives $\Delta \le \frac{1}{32}d^{-\frac{1}{1+\epsilon}} \le \frac{1}{32\sqrt{d}}$.
    
    
    Combining the preceding two display equations and averaging over all $b \in m$, we have 
    \begin{align*}
        \frac{1}{m} \sum_b\P_{\theta^{(b)}}[x_{t, d_i + b} = 1] &\le \frac{1}{m} + \frac{1}{m} \sum_{b} \sqrt{ 2^{\frac{1 + \epsilon}{\epsilon}} \Delta^{\frac{1 + \epsilon}{\epsilon}} \left(\frac{d}{m}\right)^{\frac{1 - \epsilon}{\epsilon}} \E_{\theta^{(0)}}[T_{i,b}]}  \\
        \\ & \le \frac{1}{m} + \sqrt{ 2^{\frac{1 + \epsilon}{\epsilon}}\frac{1}{m} \Delta^{\frac{1 + \epsilon}{\epsilon}} \left(\frac{d}{m}\right)^{\frac{1 - \epsilon}{\epsilon}} \sum_b \E_{\theta^{(0)}}[T_{i,b}]} \tag{Jensen, $\sum_b T_{i,b} = T$ \& choice of $\Delta$}
         \le \frac{1}{m} + \frac{1}{2}.
    \end{align*}
    Averaging over all $\theta \in \Theta$, summing over $i \in [d/m]$, and recalling that $m \ge 4$, we obtain
    \begin{align*}
        \frac{1}{|\Theta|} \sum_{\theta \in \Theta} \sum_{i = 1}^{d/m} \big( 1 - \P_\theta[x_{t, d_i + \text{ind}_i(\theta)} = 1] \big) \ge \frac{d}{m} \Big(1 - \frac{1}{m} - \frac{1}{2}\Big) \ge \frac{d}{4m}.
    \end{align*}
    Hence, there exists $\theta^* \in \Theta$ such that $ \sum_{i = 1}^{d/m} \big( 1 - \P_{\theta^*}[x_{t, d_i + \text{ind}_i(\theta^*)} = 1]\big) \ge \frac{d}{4m}$. 
 Substituting into our earlier lower bound on $R_T$ and again using our choice of $\Delta$, we obtain
    \begin{align*}
        R_T(\calA, \theta^*) \ge \frac{d}{4m} \Delta T = \frac{1}{32} d^{\frac{\epsilon}{1 + \epsilon} } \left(\frac{d}{m}\right)^{\frac{\epsilon}{1 + \epsilon}} T^{\frac{1}{1 + \epsilon}}.
    \end{align*}
    Since $f(x) = \frac{x}{\log x}$ is increasing for $x \ge e$, and $m \in [4,d]$, the definition of $m$ gives the following:
    \begin{align*}
        \frac{d}{\log n} > \frac{m - 1}{\log(m - 1)} > \frac{m - 1}{\log m} \ge \frac{m - 1}{\log d}.
    \end{align*}
    Rearranging the above, we obtain $\frac{d}{m} > \frac{\log n}{\log d} \left(1 - \frac{1}{m}\right) \ge \frac{\log n}{2\log d}$, completing the proof.


